\documentclass[preprint,authoryear,12pt]{elsarticle}
\usepackage{mathrsfs}
\usepackage{graphics,graphicx,epsfig}
\usepackage[american]{babel}
\usepackage{amsmath,mathrsfs,amssymb,amsfonts}
\usepackage{algorithm,algorithmic}
\usepackage{multirow,bm}
\def \x {\bm{x}}

\def \y {\bm{y}}

\def \D {\mathcal{D}}
\def \A {\mathcal{A}}

\def \X  {\mathcal{X}}
\def \Y  {\mathcal{Y}}

\def \E  {\mathcal{E}}

\def \y {\hat{y}}
\def \X {\mathcal{X}}
\def \Y {\mathcal{Y}}

\def \D {\mathcal{D}}
\def \hD {\hat{\mathcal{D}}}

\newtheorem{theorem}{Theorem}
\newtheorem{lemma}{Lemma}

\newtheorem{corollary}{Corollary}

\newenvironment{proof}{\noindent{\it Proof: }}{\qed\medskip}

\begin{document}

\begin{frontmatter}
\title{On the Resistance of Nearest Neighbor\\
To Random Noisy Labels}

\author{Wei Gao and Bin-Bin Yang and Zhi-Hua Zhou\corref{cor1}}

\address{National Key Laboratory for Novel Software Technology\\
             Nanjing University, Nanjing 210023, China}
\cortext[cor1]{Email: zhouzh@lamda.nju.edu.cn}
\begin{abstract}
Nearest neighbor has always been one of the most appealing non-parametric approaches in machine learning, pattern recognition, computer vision, etc.  Previous empirical studies partly shows that nearest neighbor is resistant to noise, yet there is a lack of deep analysis. This work presents the finite-sample and distribution-dependent bounds on the consistency of nearest neighbor in the random noise setting. The theoretical results show that, for asymmetric noises, \(k\)-nearest neighbor is robust enough to classify most data correctly, except for a handful of examples, whose labels are totally misled by random noises. For symmetric noises, however, \(k\)-nearest neighbor achieves the same consistent rate as that of noise-free setting, which verifies the resistance of \(k\)-nearest neighbor to random noisy labels. Motivated by the theoretical analysis, we propose the Robust \(k\)-Nearest Neighbor (R\(k\)NN) approach to deal with noisy labels. The basic idea is to make unilateral corrections to examples, whose labels are totally misled by random noises, and classify the others directly by utilizing the robustness of \(k\)-nearest neighbor. We verify the effectiveness of the proposed algorithm both theoretically and empirically.
\end{abstract}


\end{frontmatter}

\section{Introduction}
The nearest neighbor \citep{Cover:Hart1967,Fix:Hodges1951} has been one of the oldest and most intuitive approaches in pattern recognition, machine learning, computer vision, etc. The basic idea is to classify each unlabeled instance by the label of its nearest neighbor (\(1\)-NN) or by the majority labels of \(k\) nearest neighbors (\(k\)-NN). Despite of simplicity, nearest neighbor takes good performance in real applications, and makes good explanations of predictions with theoretical guarantee \citep{Berlind:Urner2015,Biau:Devroye2015,Dasgupta2012,Kontorovich:Sabato:Weiss2017,Kontorovich:Weiss2014,Kpotufe2011,Kulkarni:Posner1995,Shalev-Shwartz:Ben-David2014,Wagner1971}. Empirical studies partially demonstrate the resistance of \(k\)-nearest neighbor to noise \citep{Kusner:Tyree:Weinberger:Agrawal2014,Tarlow:Swersky:Swersky:Charlin:Sutskever:Zemel2013}, whereas there is a paucity of deep understanding.

This work focuses on binary classification in the presence of random classification noises \citep{Angluin:Laird1988}, that is, each observed label has been flipped with certain probability instead of seeing the ground-truth label, and training data of each class are contaminated by samples from the other class. Generally speaking, noisy data may disturb learning process, increase sample and model complexities, and deteriorate effectiveness and quality of learned classifiers. For example, the random noise defeats all convex potential boosters \citep{Long:Servedio2010}, and support vector machines (SVMs) are sensitive to noisy labels. Many practical algorithms have been developed to tackle noisy labels \citep{Angluin:Laird1988,Kearns1993,Lawrence:Scholkopf2001,Liu:Tao2016,Natarajan:Dhillon:Ravikumar:Tewari2013,Xu:Crammer:Schuurmans2006}, most working with parametric classifiers, yet relatively few studies focus on non-parametric classifiers.

This work presents a theoretical and empirical understanding on the resistance of nearest neighbor to random noisy labels, and the main contributions can be summarized as follows:
\begin{itemize}
\item We provide the finite-sample and distribution-dependent bounds on the consistency of nearest neighbor. Our theoretical results show that, for asymmetric noises, \(k\)-nearest neighbor is robust enough to classify most data correctly, except for a handful of totally misled examples. For symmetric noises, however, \(k\)-nearest neighbor achieves the same consistent rate as that of noise-free setting, which verifies the resistance of \(k\)-nearest neighbor. We also prove the inconsistency of \(1\)-nearest neighbor even for symmetric noises.
\item Motivated by the theoretical analysis, we propose the Robust \(k\)-Nearest Neighbor (R\(k\)NN) approach to deal with noisy labels with theoretical guarantee. The basic idea is to make unilateral corrections to examples whose labels are misled totally by random noise, and classify the others simultaneously by utilizing the robustness of \(k\)-nearest neighbor. Our approach also makes use of nearest neighbor to estimate noise from corrupted datasets.
\item Extensive experiments show the effectiveness of our R\(k\)NN algorithm on benchmark datasets, and theoretical results are also verified empirically on synthetic dataset.
\end{itemize}

\noindent\textbf{Related Work}

The random noise model \citep{Angluin:Laird1988} has motivated a series of follow-up studies in the theoretical community.  The finite VC-dimension has been used to characterize the learnability in \citep{Aslam:Decatur1996,Cesa-Bianchi:Dichterman:Fischer:Shamir:Simon1999}, and \cite{Ben-David:Pal:Shalev-Shwartz2009} showed the equivalence between Littlestone dimension and learnability of online mistake bound. \cite{Kearns1993} proposed the famous statistical query (SQ) model by capturing global statistical properties of large samples rather than individual example. \cite{Kalai:Servediob2005} made theoretical analysis of boosting algorithms in the presence of random noise.

Various practical approaches have been developed to deal with noisy data during the past decades, e.g., outlier detection \citep{Barandela:Gasca2000,Brodley:Friedl1999}, re-weights of training instances \citep{Liu:Tao2016,Rebbapragada:Brodley2007,Wang:Liu:Tao2017}, perceptron-style  algorithms \citep{Bylander1994,Crammer:Dekel:Keshet:Shalev-Shwartz:Singer2006,Dredze:Crammer:Pereira2008}, robust losses \citep{Denchev:Ding:Neven:Vishwanathan2012,Masnadi-Shirazi:Vasconcelos2009,Xu:Crammer:Schuurmans2006}, unbiased losses \citep{Gao:Wang:Li:Zhou2016,Natarajan:Dhillon:Ravikumar:Tewari2013}, etc. The interested readers are also referred to the survey \citep[reference therein]{Frenay:Verleysen2014}.

Nearest neighbor has attracted much attention during the past decades \citep{Beygelzimer:Kakade:Langford2006,Cover:Hart1967,Fix:Hodges1951,Kontorovich:Sabato:Urner2016,Kontorovich:Weiss2015,Kulkarni:Posner1995,Samworth2012,Wagner1971,Wang:Jha:Chaudhuri2017}. The asymptotic consistency of nearest neighbor has been studied in \citep{Cover:Hart1967,Dasgupta2012,Devroye:Gyorfi:Krzyzak:Lugosi1994,Devroye:Gyorfi:Lugosi1996,Fix:Hodges1951,Stone1977}. It is well-known that the classification error converges to Bayes error \(R^*\) for \(k_n\)-nearest neighbor when \(k_n=o(n)\to\infty\), and to \(R^*+O(1/\sqrt{k})\) for \(k\)-nearest neighbor, and is at most \(2R^*\) for \(1\)-nearest neighbor. The consistency analysis based on finite sample has also explored in the works of \citep{Chaudhuri:Dasgupta2014,Shalev-Shwartz:Ben-David2014}.

The rest of this paper is organized as follows. Section~\ref{sec:Pre} presents some preliminaries. Section~\ref{sec:analy} provides theoretical analysis. Section~\ref{sec:alg} develops the Robust \(k\)-Nearest Neighbor (R\(k\)NN) approach. Section~\ref{sec:pf} presents detailed proofs for our main results.  Section~\ref{sec:exp} conducts empirical studies. Section~\ref{sec:con} concludes this work with future work.

\section{Preliminaries}\label{sec:Pre}
Let \(\X=[0,1]^d\) and \(\mathcal{Y}=\{0,1\}\) denote the instance and label space, respectively. Suppose that \(\D\) is an (unknown) underlying distribution over the product space \(\X\times\Y\). Let \(\D_{\X}\) be the marginal distribution over \(\X\). Denote by \(\eta(\x)=\Pr[y=+1|\x]\) conditional probability with respect to distribution \(\D\). In this work, we assume that \(\eta(\x)\) is \(L\)-Lipschitz  for some constant \(L>0\), that is,
\[
|\eta(\x)-\eta(\x')|\leq L\|\x-\x'\|.
\]
Intuitively, this assumption implies that two instances are likely to have similar labels if they are close to each other, and such assumption has been taken in classification~\citep{Cover:Hart1967,Shalev-Shwartz:Ben-David2014}. For a hypothesis \(h\colon\X\to\Y\), we define the classification error with respect to distribution \(\D\) as
\[
R_\D(h)=\Pr\nolimits_{(\x,y)\sim\D}[h(\x)\neq y]=E_{(\x,y)\sim\D}[I[h(\x)\neq y]].
\]
Here, \(I[\cdot]\) denotes the indicator function, which returns 1 if the argument is true and 0 otherwise. It is well-known \citep{Devroye:Gyorfi:Lugosi1996} that the Bayes classifier, which minimizes classification error, is given by \(h_\D^*(\x)=I[\eta(\x)\geq1/2]\), and the Bayes error is given by \(R_\D^*=E_{\x}[\min\{\eta(\x),1-\eta(\x)\}]\).

Let \(S_n=\{(\x_1,y_1), (\x_2,y_2),\ldots,(\x_n,y_n)\}\) be a training data, where each example is drawn i.i.d. from distribution \(\D\). In the random noise model, we can not observe the true labels \(y_i\) (\(i\in[n]\)), and instead, each label has been flipped with a certain probability instead of seeing the true label, i.e., each label \(y_i\) is corrupted by random noise with proportions \(\tau_+\) and \(\tau_-\). Here
\[
\tau_+=\Pr[\y_i=-1|y_i=+1]\quad \text{ and }\quad\tau_-=\Pr[\y_i=+1|y_i=-1].
\]
Throughout this work, we assume \(\tau_++\tau_-<1\) as in \citep{Blum:Mitchell1998}. Let \(\hat{S}_n=\{(\x_1,\y_1), (\x_2,\y_2),\ldots,(\x_n,\y_n)\}\) be the corrupted data, and \(\hD\) denotes the corrupted distribution from true distribution \(\D\) by random noises with proportions \(\tau_+\) and \(\tau_-\). Let \(\hat{\eta}(\x)=\Pr[\y=+1|x]\) be the corrupted conditional probability w.r.t. distribution \(\hD\). It is easy to get the relationship between \(\eta(\x)\) and \(\hat{\eta}(\x)\) as follows:
\begin{equation}\label{eq:tmp000}
\hat{\eta}(\x)=(1-\tau_+-\tau_-)\eta(\x)+\tau_-.
\end{equation}
We further have \(|\hat{\eta}(\x)-\hat{\eta}(\x')|=(1-\tau_+-\tau_-)|\eta(\x)-\eta(\x')|\), and if \(\eta(\x)\) is \(L\)-Lipschitz, then \(\hat{\eta}(\x)\) is \((1-\tau_--\tau_+)L\)-Lipschitz, i.e.,
\begin{equation}\label{lem:relation}
|\hat{\eta}(\x)-\hat{\eta}(\x')|\leq (1-\tau_+-\tau_-)L\|\x-\x'\|.
\end{equation}

It is interesting to discuss the Tsybakov noise condition \citep{Tsybakov2004}, i.e., for some finite \(C_0>0\) and \(\lambda>0\), we have
\[
\Pr[|\eta(\x)-1/2|\leq t]\leq C_0 t^\lambda,
\]
which presents faster convergence rate \citep{Tsybakov2004} in noise-free setting. It is noteworthy that this assumption is over true distribution \(\D\), while the random noise setting does not make any assumption over \(\D\). From Eqn.~\eqref{eq:tmp000}, we have the corrupted conditional probability \(\hat{\eta}(\x)\) such that
\[
\Pr[|\hat{\eta}(\x)-1/2+(\tau_+-\tau_-)/2|\leq t]\leq C_0 t^\lambda/(1-\tau_+-\tau_-)^\lambda
\]
if the Tsybakov noise condition holds for distribution \(\D\). This implies that Tsybakov noise condition can not be guaranteed for asymmetric noise even if the true distribution \(\D\) does. We consider the general random noise model without assumption over distribution \(\D\) in this work, and it is interesting to study random noise model for distribution \(\D\) with Tsybakov noise condition.

Let \([n]=\{1,2,\ldots,n\}\) for integer \(n\geq0\). Denote by B\((p)\) a Bernoulli distribution of parameter \(p\in[0,1]\), and \(y\sim \text{B}(p)\) represents that random variable \(y\) is drawn from Bernoulli distribution B\((p)\). We do not know the true data \(S_n\), noise proportions \(\tau_+\) and \(\tau_-\), and distributions \(\D\) and \(\hat{\D}\) in practice, and what we can observe is a corrupted data \(\hat{S}_n\). The goal of this work is to learn a hypothesis \(h_{\hat{S}_n}\) with lower classification error over true distribution \(\D\), but it is trained on the corrupted data \(\hat{S}_n\).

\section{Theoretical Analysis}\label{sec:analy}
Given corrupted training data \(\hat{S}_n\) and instance \(\x\in\X\), let
\[
\pi_1(\x), \pi_2(\x), \ldots,\pi_n(\x)
\]
be a reordering of \(\{1,2,\ldots,n\}\) according to their distances to \(\x\), that is, \(\|\x-\x_{\pi_i(\x)}\|\leq \|\x-\x_{\pi_{i+1}(\x)}\| \text{ for }i<n\). For \(k\)-nearest neighbor algorithm, the output hypothesis is given by
\[
h^k_{\hat{S}_n}(\x)=\left\{
\begin{array}{ll}
1 & \text{ for } \sum_{i=1}^k \hat{y}_{\pi_i(\x)}\geq k/2,\\
0 & \text{ otherwise.}
\end{array}
\right.
\]
Let \(\mathcal{E}^b =\{\x\in\X\colon \eta(\x)= 1/2\}\) denote the boundary set of Bayes's classifier with respect to distribution \(\D\), and we further introduce, for \(\Delta\geq0\),
\begin{eqnarray*}
\mathcal{E}^+_\Delta&=&\{\x\in\X\colon \eta(\x)> 1/2,~~\hat{\eta}(\x)\geq 1/2 +\Delta \}, \\
\mathcal{E}^-_\Delta&=& \{\x\in\X\colon \eta(\x)< 1/2,~~\hat{\eta}(\x)\leq 1/2 -\Delta \},
\end{eqnarray*}
where \(\mathcal{E}^+_\Delta\) and \(\mathcal{E}^-_\Delta\) show the most correctly predictive sets of positive and negative instances in the noise setting, respectively. Denote by
\[
\A_\Delta=\X\setminus (\E^+_\Delta \cup \E^-_\Delta\cup \E^b)
\]
where the labels are relatively hard to be predicted correctly. It is important to introduce the set
\begin{equation}\label{eq:A0}
\A_0= \X\setminus (\E^+_0 \cup \E^-_0 \cup \E^b)=\{\x\in\X\colon (\eta(\x)-1/2)(\hat{\eta}(\x)-1/2)<0\},
\end{equation}
where the labels are totally misled by random noise. We first present the consistency analysis of \(k\)-nearest neighbour in the random noise setting.
\begin{theorem}\label{theorem:noiknn}
Let \(\hat{S}_n\) be a corrupted sample with noise proportions \(\tau_-\) and \(\tau_+\).  Let \(h^k_{\hat{S}_n}\) be the output hypothesis of applying \(k\)-nearest neighbor to \(\hat{S}_n\). We have
\begin{eqnarray*}
{E}_{\hat{S}_n\sim\hD^n}[R_\D(h^k_{\hat{S}_n})] &\leq& R_D^*+ \Pr\nolimits_{\x\sim\D_\X}[\x\in\A_\Delta] \\
&& +\frac{1}{\sqrt{k}}+2\Big((1-\tau_+-\tau_-)L/\sqrt{d}\Big)^{\frac{d}{1+d}} \Big(\frac{k}{n}\Big)^{\frac{1}{1+d}}
\end{eqnarray*}
where \(\Delta=\max\left\{2\sqrt{d}((1-\tau_+-\tau_-)L)^\frac{d}{1+d}(k\sqrt{d}/n)^\frac{1}{1+d}, \sqrt{\log k/ k}\right\}\).
\end{theorem}

Notice that the hypothesis \(h^k_{\hat{S}_n}\) of \(k\)-nearest neighbor is trained on the corrupted sample \(\hat{S}_n\), while  \(R_\D(h^k_{\hat{S}_n})= \Pr_{(\x,y)\sim\D}[h^k_{\hat{S}_n}(\x)\neq y]\) denotes the classification error over the true distribution \(\D\). The detailed proof is given in Section~\ref{sec:pf:theorem:noiknn}. Based on this theorem, we have
\begin{equation}\label{eq:bias}
{E}_{\hat{S}_n\sim\hD^n}[R_\D(h^k_{\hat{S}_n})] \leq  R_D^*+\Pr\nolimits_{\x\sim\D_\X}[\x\in\A_0]
\end{equation}
for \(k=k(n)\to \infty\)  and \(k/n\to0\) as \(n\to\infty\).

As can be seen, the classification error of \(k\)-nearest neighbor is biased at most \(\Pr_{\x\sim\D_X}[\x\in\A_0]\) from the Bayes error \(R_D^*\) in the asymptotic convergence. Hence, \(k\)-nearest neighbor is robust enough to classify most examples correctly, except for examples in \(\A_0\), whose labels are totally misled. This motivates us to design effective strategy to tackle \(\A_0\), which will be shown in Section~\ref{sec:alg}.

For symmetric noises \(\tau_+=\tau_-\), we have \(\A_0=\emptyset\) from Eqn.~\eqref{eq:tmp000}, which proves the consistency of \(k\)-nearest neighbor. Actually, we can further provide a stronger theorem, and the detailed proof is presented in Section~\ref{sec:pf:thm:noiknn}.
\begin{theorem}\label{thm:uni-noise}
For \(k\geq8\) and \(\tau_-=\tau_+=\tau\), we have
\[
E_{\hat{S}_n\sim\D^n}[R(h^k_{\hat{S}_n})]\leq R_\D^*+\frac{2 R_\D^*}{\sqrt{k}}
+\frac{2\tau}{(1-2\tau)\sqrt{k}}+5\max\{L,\sqrt{L}\}\sqrt{d} \Big(\frac{k}{n}\Big)^{\frac{1}{1+d}}.
\]
\end{theorem}
This theorem shows that
\[
E_{\hat{S}_n\sim\D^n}[R(h^k_{\hat{S}_n})]\to R_\D^* \text{ for }k=k(n)\to\infty \text{ and }k/n\to0 \text{ as }n\to\infty;
\]
and we also have
\[
E_{\hat{S}_n\sim\D^n}[R(h^k_{\hat{S}_n})]\to R_\D^*+O(1/\sqrt{k})\text{ for constant } k \text{ as }n\to\infty.
\]
Hence, the consistent rate of \(k\)-nearest neighbors in the symmetric noise setting are the same as that of noise-free setting \citep{Biau:Devroye2015,Chaudhuri:Dasgupta2014,Dasgupta2012,Devroye1981,Fix:Hodges1951,Shalev-Shwartz:Ben-David2014,Stone1977}, which proves the resistance of \(k\)-nearest neighbor, particularly for large \(k\).

From Theorem~\ref{thm:uni-noise}, we present consistency analysis for \(k\)-nearest neighbor in the noise-free setting:
\begin{corollary}\label{coro:noise-free}
For \(k\geq8\) and \(\tau_0=\tau_+=0\), we have
\[
E_{{S}_n\sim\D^n}[R(h^k_{{S}_n})]\leq R_\D^*+\frac{2 R_\D^*}{\sqrt{k}}+5\max\{L,\sqrt{L}\}\sqrt{d} \Big(\frac{k}{n}\Big)^{\frac{1}{1+d}}.
\]
\end{corollary}
This corollary improves the work of \citep[Theorem 19.5]{Shalev-Shwartz:Ben-David2014}, which can be written (with our notations) as
\[
E_{{S}_n\sim\D^n}[R(h^k_{{S}_n})]\leq R_\D^*+\frac{2\sqrt{2} R_\D^*}{\sqrt{k}}+ \frac{6L\sqrt{d}+k}{n^{\frac{1}{1+d}}} \text{ for }k\geq10.
\]
As can be seen, the above shows the consistency of nearest neighbor as \(k/n^{1/(1+d)}\to0\), while Corollary~\ref{coro:noise-free} presents tighter consistent rate of nearest neighbor as \(k/n\to0\).

We finally study the inconsistency of \(1\)-nearest neighbor even for symmetric noise as follows.
\begin{theorem}\label{thm:nearneighbor}
Let \(h^1_{\hat{S}_n}\) be the output hypothesis of applying \(1\)-nearest neighbor to \(\hat{S}_n\). For \(\tau_-=\tau_+=\tau\), we have
\[
(1-2\tau)\Big(R_\D^*-\frac{3L\sqrt{d}}{2n^{{1}/(d+1)}}\Big)\leq E_{\hat{S}_n\sim\hD^n}[R_\D(h^1_{\hat{S}_n})]-\tau\leq (1-2\tau)\Big(2R_\D^*+\frac{3 L\sqrt{d}}{2n^{{1}/(d+1)}}\Big).
\]
\end{theorem}
The proof is presented in Section~\ref{sec:pf:thm:nearneighbor}. We also notice that the inconsistency of 1-nearest neighbor has been well-studied in the noise-free setting \citep{Biau:Devroye2015,Cover:Hart1967,Shalev-Shwartz:Ben-David2014}. Theorem~\ref{thm:nearneighbor} is easier to show the influence of random noises from
\[
R_\D^*+(1-2R_\D^*)\tau\leq E_{\hat{S}_n\sim\hD^n}[R_\D(h^1_{\hat{S}_n})]\leq 2R_\D^*+(1-4 R_\D^*)\tau  \ \text{ as }\  n\to\infty.
\]

\section{The R\(k\)NN Approach}\label{sec:alg}
Motivated from the preceding theoretical analysis, \(k\)-nearest neighbor is robust enough to tackle most data except for examples in \(\A_0\), whose labels are totally misled by random noises. Therefore, our basic idea is to develop effective strategies to classify examples in \(\A_0\) correctly, and classify the others simultaneously by utilizing the robustness of \(k\)-nearest neighbor.

From Eqns.~\eqref{eq:tmp000} and \eqref{eq:A0}, we have
\[
\A_0=\{\x\in\X\colon (\eta(\x)-1/2)(\hat{\eta}(\x)-1/2)<0\},
\]
and
\[
\hat{\eta}(\x)-1/2=(1-\tau_+-\tau_-)(\eta(\x)-1/2)+(\tau_--\tau_+)/2.
\]
For \(\tau_->\tau_+\), this follows that
\[
\A_0= \{\x\in\X\colon 0< \hat{\eta}(\x)-1/2< \tau_-/2-\tau_+/2\}
\]
because we have \(0<\hat{\eta}(\x)-1/2< \tau_-/2-\tau_+/2\) if \(\eta(\x)<1/2\) and \(\x\in\A_0\); we also have \(\hat{\eta}(\x)>1/2\) if \(\eta(\x)>1/2\), which implies \(\x\notin \A_0\). In a similar manner, we have, for \(\tau_-<\tau_+\),
\[
\A_0= \{\x\in\X\colon \tau_-/2-\tau_+/2< \hat{\eta}(\x)-1/2< 0\}.
\]
This motivates us to make unilateral corrections to corrupted labels in \(\A_0\). Specifically, we predict the label as \(0\) for instance \(\x\in\X\) if \(\tau_->\tau_+\) and \(\hat{\eta}(\x)-1/2\in(0,\tau_-/2-\tau_+/2)\); and predict as \(1\) for instance \(\x\in\X\) if \(\tau_-<\tau_+\) and \(\hat{\eta}(\x)-1/2\in(\tau_-/2-\tau_+/2, 0)\).

The conditional probability \(\hat{\eta}(\x)\) is unknown in practice, and what we can observe is a training data \(\hat{S}_n\). Let \((\x_{\pi_1(\x)},\y_{\pi_1(\x)}), \ldots,(\x_{\pi_k(\x)},\y_{\pi_k(\x)})\) be \(k\) nearest neighbors of instance \(\x\), and we approximate \(\hat{\eta}(\x)\approx \sum\nolimits_{i=1}^k \y_{\pi_i(\x)}/k\), where \(k\) is called \textit{predictive} parameter.

We call such method Robust \(k\)-Nearest Neighbor (R\(k\)NN). Under the prior knowledge of noise proportions \(\tau_+\) and \(\tau_-\), we can prove the consistency of R\(k\)NN as follows.

\begin{theorem}\label{thm:ourRkNN}
Let \(\hat{S}_n\) be a corrupted sample with noise proportions \(\tau_-\) and \(\tau_+\).  Let \(h^{rk}_{\hat{S}_n}\) be the output hypothesis of applying our R\(k\)NN approach to \(\hat{S}_n\). For constant \(k\), we have
\[
\mathop{E}_{\hat{S}_n\sim\hD^n}[R_\D(h^{rk}_{\hat{S}_n})] \to R_D^*+O(1/\sqrt{k})
\]
as \(n\to\infty\); we also have, for \(k=k(n)\to\infty\) and \(k/n\to0\) as \(n\to\infty\),
\[
\mathop{E}_{\hat{S}_n\sim\hD^n}[R_\D(h^{rk}_{\hat{S}_n})] \to R_D^*.
\]
\end{theorem}

As we can see, the R\(k\)NN approach achieves the same consistent rate as that of traditional \(k\)-nearest neighbor in the noise-free setting. Notice that the prior knowledge of \(\tau_+\) and \(\tau_-\) is also necessary for the proof of consistency of ERMs as in the work of \citep{Natarajan:Dhillon:Ravikumar:Tewari2013}. The detailed proof is presented in Section~\ref{sec:pf:thm:ourRkNN}.

How to estimate noise proportions \(\tau_+\) and \(\tau_-\) from a corrupted sample has been an interesting and well-studied problem \citep{Liu:Tao2016,Menon:Rooyen:Ong:Williamson2015,Ramaswamy:Scott:Tewari2016,Scott:Blanchard:Handy2013}. We follow the idea of conditional probability as in \citep{Liu:Tao2016,Menon:Rooyen:Ong:Williamson2015}, but introduce another \(k'\)-nearest neighbor rather than learning the corrupted conditional distribution.

\begin{algorithm}
\caption{Robust \(k\)-Nearest Neighbor (R\(k\)NN)}\label{alg1}

\noindent\textbf{Input}: Corrupted sample \(\hat{S}_n=\{(\x_1,\y_1), \ldots,(\x_n,\y_n)\}\), new instance \(\x\in\X\), predictive parameter  \(k\) and noise parameter \(k'\)

\begin{algorithmic}[1]
\STATE Calculate \(\hat{\eta}(\x_j)\approx \sum\nolimits_{i=0}^{k'} \y_{\pi_i(\x_j)}/(k'+1)\) for \(j\in [n]\) by \(k'\)-nearest neighbor
\STATE Estimate noise proportions \(\hat{\tau}_+\) and \(\hat{\tau}_-\) from Eqn.~\eqref{eq:noise:est}
\STATE Calculate \(\hat{\eta}(\x)\approx \sum\nolimits_{i=1}^k \y_{\pi_i(\x)}/k\), where \(\x_{\pi_1(\x)}, \ldots,\x_{\pi_k(\x)}\) are the \(k\) nearest neighbors of \(\x\)
\STATE Set \(y=I[\hat{\eta}(\x)\geq 1/2]\)
\IF{\(\hat{\tau}_->\hat{\tau}_+\) and \(\hat{\eta}(\x)-1/2\in (0,\hat{\tau}_-/2-\hat{\tau}_+/2)\)}
\STATE Update \(y=0\)
\ENDIF
\IF{\(\hat{\tau}_-<\hat{\tau}_+\) and \(\hat{\eta}(\x)-1/2\in(\hat{\tau}_-/2-\hat{\tau}_+/2,0)\)}
\STATE Update \(y=1\)
\ENDIF
\end{algorithmic}

\textbf{Output}: the predicted label \(y\)~~~~~~~~~~~~~~~~~~~~~~~~~~~~~~~~~~~~~~~~
\end{algorithm}

Specifically, let \((\x_{\pi_1(\x_j)}, \y_{\pi_1(\x_j)}), \ldots, (\x_{\pi_{k'}(\x_j)},\y_{\pi_{k'}(\x_j)})\) be the \(k'\) nearest neighbors of example \((\x_j,\y_j)\). We approximate \(\hat{\eta}(\x_j)\approx \sum\nolimits_{i=0}^{k'} \y_{\pi_i(\x_j)}/(k'+1)\), where \(\y_{\pi_0(\x_j)}=\y_j\) and \(k'\) is called \textit{noise} parameter. As in the works of \citep{Liu:Tao2016,Menon:Rooyen:Ong:Williamson2015},  the estimated noise proportions \(\hat{\tau}_+\) and \(\hat{\tau}_-\) can be given, respectively, by
\begin{equation}\label{eq:noise:est}
\hat{\tau}_+=\min\nolimits_{j\in[n]} \{\hat{\eta}(\x_j)\} \quad\text{ and }\quad \hat{\tau}_-=\min\nolimits_{j\in[n]} \{1-\hat{\eta}(\x_j)\}.
\end{equation}

Algorithm~\ref{alg1} presents the detailed description of the proposed R\(k\)NN method, and it can be further simplified to be traditional \(k\)-nearest neighbor classification when \(\hat{\tau}_+=\hat{\tau}_-\).

\section{Proofs}\label{sec:pf}
This section present the detailed proofs for our main results.

\subsection{Proof of Theorem~\ref{theorem:noiknn}}\label{sec:pf:theorem:noiknn}
The proof techniques are partially inspired by the works of \citep{Chaudhuri:Dasgupta2014,Shalev-Shwartz:Ben-David2014}. By the union bounds, we have
\begin{eqnarray}
\lefteqn{\Pr_{(\x,y)\sim\D}[h^k_{\hat{S}_n}(\x)\neq y]}\nonumber\\
&=&\Pr_{(\x,y)\sim\D}[h^k_{\hat{S}_n}(\x)\neq y,\x\in \A_\Delta] +\Pr_{(\x,y)\sim\D}[h^k_{\hat{S}_n}(\x)\neq y,\x\notin \A_\Delta] \nonumber\\
&\leq&\Pr_{\x\sim\D_\X}[\x\in \A_\Delta] +\Pr_{(\x,y)\sim\D}[h^k_{\hat{S}_n}(\x)\neq y,\x\notin \A_\Delta].\label{eq:tmp0}
\end{eqnarray}
Fixed \(\mu>0\), let \(C_1,\ldots,C_r\) be the cover of \(\X=[0,1]^d\) with boxes of length \(\mu\), and we have \(r=(1/\mu)^d\). Denote by two random events
\begin{eqnarray*}
\Gamma_1(\x,\x')&=&\{\text{there exists some } C_i\text{ such that } \x\in C_i\text{ and }\x'\in C_i\},\\
\Gamma_2(\x,\x')&=&\{\text{we have }\x\notin C_i\text{ or }\x'\notin C_i \text{ for every }C_i\}.
\end{eqnarray*}
Based on total probability theorem, we have
\begin{eqnarray}
\lefteqn{\Pr_{(\x,y)\sim\D}[h^k_{\hat{S}_n}(\x)\neq y,\x\notin \A_\Delta]}\nonumber\\
&=&\Pr_{(\x,y)\sim\D}[h^k_{{\hat{S}}_n}(\x)\neq y, \x\notin\A_\Delta\big|\Gamma_2(\x,\x_{\pi_{k}(\x)})] \Pr_{(\x,y)\sim\D}[\Gamma_2(\x,\x_{\pi_{k}(\x)})]\nonumber\\
&&+\Pr_{(\x,y)\sim\D}[h^k_{{\hat{S}}_n}(\x)\neq y, \x\notin\A_\Delta\big|\Gamma_1(\x,\x_{\pi_{k}(\x)})] \Pr_{(\x,y)\sim\D}[\Gamma_1(\x,\x_{\pi_{k}(\x)})]\nonumber\\
&\leq& \Pr_{(\x,y)\sim\D}[\Gamma_2(\x,\x_{\pi_{k}(\x)})]
+\Pr_{(\x,y)\sim\D}\big[h^k_{{\hat{S}}_n}(\x)\neq y, \x\notin\A_\Delta\big|\Gamma_1(\x,\x_{\pi_{k}(\x)}) \big].\label{eq:tmp2}
\end{eqnarray}
According to Lemma~\ref{lem:knn:tmp1}, the first term in the above can be upper bounded by
\begin{equation}\label{eq:tmp33}
E_{\hat{S}_n\sim\hD^n}\big[\Pr_{(\x,y)\sim\D}[\Gamma_2(\x,\x_{\pi_{k}(\x)})]\big]\leq {2rk}/{n}.
\end{equation}
For the second term, we fix \(\x_1,\x_2,\ldots,\x_n\) and \(\x\notin \A_\Delta\). Assume that \(\x_1,\ldots,\x_k\) are \(k\)-nearest neighbors of \(\x\). Let \(\hat{\eta}(\x_1),\ldots,\hat{\eta}(\x_k)\) be conditional probabilities  w.r.t. corrupted distribution \(\hD\), and set \(\hat{p}=\sum_{i=1}^k \hat{\eta}(\x_i)/k\).

\noindent If \(\x\in\E^b\), then we have
\[
\Pr_{y\sim\text{B}(\eta(\x))}[h^k_{{\hat{S}}_n}(\x)\neq y]=1/2=\min\{\eta(\x),1-\eta(\x)\}.
\]
If \(\x\in\E^-_\Delta\), then we set \(\Delta=2(1-\tau_+-\tau_-)L\mu\sqrt{d}\), and have
\begin{equation}\label{eq:tmp3}
\hat{p}<1/2-\Delta/2
\end{equation}
from \(|\hat{\eta}(\x)-\hat{p}|\leq \Delta/2\) by Eqn.~\eqref{lem:relation}. We also have
\begin{eqnarray*}
\lefteqn{\Pr_{y\sim\text{B}(\eta(\x))}[h^k_{{\hat{S}}_n}(\x)\neq y]} \\
&=&\eta(\x)I[h^k_{{\hat{S}}_n}(\x)=0]+(1-\eta(\x))I[h^k_{{\hat{S}}_n}(\x)=1]\\
&=&\eta(\x)+(1-2\eta(\x))I[h^k_{{\hat{S}}_n}(\x)=1].
\end{eqnarray*}
For \(k\)-nearest neighbor, it holds that
\[\
I[h^k_{{\hat{S}}_n}(\x)=1]=I\Big[\sum_{i=1}^k\frac{\hat{y_i}}{k}>1/2\Big].
\]
Combining with Eqn.~\eqref{eq:tmp3} and Chernoff's bounds, we have
\[
\Pr_{\hat{y}_1\sim\text{B}(\hat{\eta}(\x_1)),\ldots, \hat{y}_k\sim\text{B}(\hat{\eta}(\x_k))}\Big[\sum_{i=1}^k\frac{\hat{y_i}}{k} -\hat{p}>\frac{1}{2}-\hat{p}\Big]\leq \exp(-k\Delta^2/2).
\]
This follows that
\[
\mathop{E}_{\hat{y}_1\sim\text{B}(\hat{\eta}(\x_1)),\ldots, \hat{y}_k\sim\text{B}(\hat{\eta}(\x_k))} \Big[\Pr_{y\sim\text{B}(\eta(\x))}[I[h^k_{{\hat{S}}_n}(\x)\neq y]]\Big]
\leq \eta(\x)+ \exp(-k\Delta^2/2),
\]
and it is noteworthy of \(\eta(\x)=\min\{\eta(\x),1-\eta(\x)\}\) for \(\x\in\E_\Delta^-\).

Similarly, we prove that, for \(\x\in\E_\Delta^+\),
\begin{multline}\label{eq:pft}
\mathop{E}_{\hat{y}_1\sim\text{B}(\hat{\eta}(\x_1)),\ldots, \hat{y}_k\sim\text{B}(\hat{\eta}(\x_k))} \Big[\Pr_{y\sim\text{B}(\eta(\x))}[I[h^k_{{\hat{S}}_n}(\x)\neq y]]\Big]\\
\leq 1-\eta(\x)+ \exp(-k\Delta^2/2),
\end{multline}
and it is noteworthy of \(1-\eta(\x)=\min\{\eta(\x),1-\eta(\x)\}\) for \(\x\in\E_\Delta^+\).

Combining with Eqns.~\eqref{eq:tmp0}-\eqref{eq:pft}, we have
\begin{eqnarray*}
\lefteqn{{E}_{\hat{S}_n\sim\hD^n}[R(h^k_{\hat{S}_n})]}\\
&=& {E}_{\hat{S}_n\sim\hD^n}[\Pr_{(\x,y)\sim\D}[h^k_{\hat{S}_n}(\x)\neq y]\\
&\leq& \Pr_{\x\sim\D_\X}[\x\in \A_\Delta]+\exp(-k\Delta^2/2)+R^*_\D+{2rk}/{n}
\end{eqnarray*}
where \(\Delta=2(1-\tau_+-\tau_-)L\mu\sqrt{d}\). From \(r=1/\mu^d\), we set
\[
\mu=\max\Big\{\frac{\sqrt{\log k/k}}{2(1-\tau_+-\tau_-)L\sqrt{d}},\Big(\frac{k\sqrt{d}}{n(1-\tau_+-\tau_-)L}\Big)^{\frac{1}{d+1}}\Big\}
\]
which completes the proof of Theorem~\ref{theorem:noiknn} by simple calculations.\qed

\begin{lemma}\citep[Lemma 19.6]{Shalev-Shwartz:Ben-David2014}\label{lem:knn:tmp1}
Denote by \(C_1, C_2, \ldots, C_r\) a collection of subsets over some domain \(\X\). Let \(S\) be a data of \(m\) samples drawn i.i.d. according to distribution \(\D\). Then, for every \(k\geq2\), we have \({E}_{S\sim\D^m}[\sum_{i\colon |C_i\cap S|<k} \Pr[C_i]]\leq {2rk}/{m}\).
\end{lemma}

\subsection{Proofs of Theorem~\ref{thm:uni-noise}}\label{sec:pf:thm:noiknn}
Fixed \(\mu>0\), let \(C_1,\ldots,C_r\) be the cover of instance space \(\X\) using boxes of length \(\mu\), as the proof in Section~\ref{sec:pf:theorem:noiknn}. We have \(r=(1/\mu)^d\), and denote by the events
\begin{eqnarray*}
\Gamma_1(\x,\x')&=&\{\text{there exists a } C_i\text{ such that } \x\in C_i\text{ and }\x'\in C_i\},\\
\Gamma_2(\x,\x')&=&\{ \text{we have } \x\notin C_i\text{ or }\x'\notin C_i \text{ for every }C_i\}.
\end{eqnarray*}
Based on the total probability theorem, we have
\begin{eqnarray*}
\lefteqn{\Pr_{(\x,y)\sim\D}[h^k_{{\hat{S}}_n}(\x)\neq y]}\\
&=&\Pr_{(\x,y)\sim\D}\left[h^k_{{\hat{S}}_n}(\x)\neq y\left|\Gamma_1(\x,\x_{\pi_{k}(\x)}) \right.\right] \Pr_{(\x,y)\sim\D}\left[\Gamma_1(\x,\x_{\pi_{k}(\x)})\right]\\
&&+\Pr_{(\x,y)\sim\D}\left[h^k_{{\hat{S}}_n}(\x)\neq y\left|\Gamma_2(\x,\x_{\pi_{k}(\x)})\right.\right] \Pr_{(\x,y)\sim\D}\left[\Gamma_2(\x,\x_{\pi_{k}(\x)})\right]\\
&\leq&\Pr_{(\x,y)\sim\D}\left[h^k_{{\hat{S}}_n}(\x)\neq y\left|\Gamma_1(\x,\x_{\pi_{k}(\x)}) \right.\right]   + \Pr_{(\x,y)\sim\D}\left[\Gamma_2(\x,\x_{\pi_{k}(\x)})\right].
\end{eqnarray*}
This follows that
\begin{eqnarray}
\lefteqn{E_{\hat{S}_n\sim\hD^n}[R_\D(h^k_{{\hat{S}}_n})]}\nonumber\\
&=&E_{\hat{S}_n\sim\hD^n}\left[\Pr_{(\x,y)\sim\D}[h^k_{{\hat{S}}_n}(\x)\neq y]\right]\nonumber\\
&\leq& \frac{2rk}{n} +E_{\hat{S}_n\sim\hD^n}\left[\Pr_{(\x,y)\sim\D}\left[I[h^k_{{\hat{S}}_n}(\x)\neq y]\left|\Gamma_1(\x,\x_{\pi_{k}(\x)}) \right.\right]\right],\label{eq:knn:tmp1}
\end{eqnarray}
where the inequality holds from the following inequality, by Lemma~\ref{lem:knn:tmp1},
\[
E_{\hat{S}_n\sim\hD^n}\left[\Pr_{(\x,y)\sim\D}\left[\Gamma_2(\x,\x_{\pi_{k}(\x)})\right]\right]=E_{\hat{S}_n\sim\hD^n}\Big[\sum_{i\colon C_i\cap \hat{S}_n=\emptyset} P[C_i]\Big]\leq \frac{2rk}{n}.
\]
To upper bound Eqn.\eqref{eq:knn:tmp1}, we first fix the training instances \(\x_1,\x_2,\ldots,\x_n\) and instance \(\x\), and assume that \(\x_1,\ldots,\x_k\) are the \(k\)-nearest neighbors, i.e., \(\|\x_i-\x\|\leq\mu\sqrt{d}\) for \(i\in[k]\). Let \(\eta(\x_1),\ldots,\eta(\x_k)\) be the conditional probability w.r.t. distribution \(\D\), and let \(\hat{\eta}(\x_1),\ldots,\hat{\eta}(\x_k)\) be the conditional probability w.r.t. the corrupted distribution \(\D\). We set \(p=\sum_{i=1}^k \eta(\x_i)/k\) and \(\hat{p}=\sum_{i=1}^k \hat{\eta}(\x_i)/k\). This follows
\begin{equation}\label{eq:knn:tmp3}
(1-2\tau)p=\hat{p}-\tau
\end{equation}
because \(\hat{\eta}(\x_i)=\eta(\x_i)(1-\tau)+\tau(1-\eta(\x_i))=\eta(\x_i)+\tau-2\tau\eta(\x_i)\) for every \(i\in[k]\). We also have
\begin{eqnarray}
\lefteqn{\Pr_{y\sim\text{B}(\eta(\x))}[h^k_{{\hat{S}}_n}(\x)\neq y]}\nonumber\\
&=&\eta(\x)I[h^k_{{\hat{S}}_n}(\x)\neq 1]+(1-\eta(\x))I[h^k_{{\hat{S}}_n}(\x)\neq 0]\nonumber\\
&\leq& pI[h^k_{{\hat{S}}_n}(\x)\neq 1]+(1-p)I[h^k_{{\hat{S}}_n}(\x)\neq 0]+|\eta(\x)-p|\nonumber\\
&\leq& \Pr_{y\sim\text{B}(p)}[I[h^k_{{\hat{S}}_n}(\x)\neq y]] +|\eta(\x)-p|.\label{eq:knn:tmp2}
\end{eqnarray}
This follows that, from Lemma~\ref{lem:key:knn} and Eqn.~\eqref{eq:knn:tmp3},
\begin{eqnarray}
\lefteqn{E_{\hat{y}_1\sim\text{B}(\hat{\eta}(\x_1)),\ldots,\hat{y_k}\sim\text{B}(\hat{\eta}(\x_k))}\left[\Pr_{y\sim \text{B}(p)}[h^k_{{\hat{S}}_n}(\x)\neq y]\right]}\nonumber \\
&\leq&\left(1+\sqrt\frac{2}{{k}}\right)\Pr_{y\sim\text{B}(p)}[y\neq I[\hat{p}>1/2]]+ \frac{\sqrt{2}\tau}{\sqrt{k}(1-2\tau)}\nonumber\\
&=&\left(1+\sqrt\frac{2}{{k}}\right)\Pr_{y\sim\text{B}(p)}[y\neq I[{p}>1/2]]+ \frac{\sqrt{2}\tau}{\sqrt{k}(1-2\tau)}\label{eq:knn:tmp4}
\end{eqnarray}
where the last equality holds from Lemma~\ref{lem:knn:tmp4}. We further have
\[
\Pr_{y\sim\text{B}(p)}[y\neq I[{p}>1/2]]=\min\{p,1-p\}\leq \min\{\eta(\x),1-\eta(\x)\}+|p-\eta(\x)|,
\]
which implies, by combining with Eqns.~\eqref{eq:knn:tmp1},  \eqref{eq:knn:tmp2} and \eqref{eq:knn:tmp4},
\begin{eqnarray*}
\lefteqn{E_{\hat{S}_n\sim\hD^n}[R_\D(h^k_{{\hat{S}}_n})]}\\
&\leq& \left(1+\sqrt{\frac{2}{k}}\right)R^*_\D+\frac{\sqrt{2}\tau}{\sqrt{k}(1-2\tau)} +\frac{2rk}{n}+\left(2+\sqrt{\frac{2}{k}}\right) |p-\eta(\x)|\\
&\leq&  \left(1+\sqrt{\frac{2}{k}}\right)R^*_\D+\frac{\sqrt{2}\tau}{\sqrt{k}(1-2\tau)} +\frac{2k\mu^{-d}}{n}+\left(2+\sqrt{\frac{2}{k}}\right)L\mu\sqrt{d}.
\end{eqnarray*}
By setting \(\mu=\left(2k\sqrt{d}/(nL(2+\sqrt{2/k}))\right)^{\frac{1}{1+d}}\), this follows that
\begin{eqnarray*}
E_{\hat{S}_n\sim\hD^n}[R_\D(h^k_{{\hat{S}}_n})]&\leq&\left(1+\sqrt{\frac{2}{k}}\right)R^*_\D+ \frac{\sqrt{2}\tau}{\sqrt{k}(1-2\tau)}\\
&&+\left(2+\sqrt{\frac{2}{k}}\right)L\sqrt{d}\left(1+\frac{1}{d}\right) \left(\frac{2\sqrt{d}k}{(2+\sqrt{2/k})nL}\right)^{\frac{1}{1+d}}\ .
\end{eqnarray*}
From Lemma~\ref{lem:knn:tmp0}, we have
\begin{eqnarray*}
\lefteqn{\left(2+\sqrt{\frac{2}{k}}\right)L\sqrt{d}\left(1+\frac{1}{d}\right) \left(\frac{2\sqrt{d}k}{(2+\sqrt{2/k})nL}\right)^{\frac{1}{1+d}}} \\
&&\leq \left(4+2\sqrt{\frac{2}{k}}\right)L\sqrt{d} \left(\frac{2k}{(2+\sqrt{2/k})nL}\right)^{\frac{1}{1+d}}\\
&& \leq 5L\sqrt{d} \left(\frac{k}{nL}\right)^{\frac{1}{1+d}}\leq 5\max\{L,\sqrt{L}\}\sqrt{d} \left(\frac{k}{n}\right)^{\frac{1}{1+d}}
\end{eqnarray*}
for \(d\geq1\) and \(k\geq8\). This completes the proof of Theorem~\ref{thm:uni-noise}.\qed

\begin{lemma}\label{lem:knn:tmp4}
For \(p,\hat{p}\in[0,1]\) and \(\tau\in[0,1/2)\), let \(\hat{p}=p+\tau-2p\tau\). We have
\[
p<1/2\quad\text{ if and only if }\quad \hat{p}<1/2
\]
\end{lemma}
\begin{lemma}\label{lem:knn:tmp0}
For \(t\geq1\), we have
\[
\left(1+{1}/{t}\right)t^{\frac{1}{2(t+1)}}\leq 2.
\]
\end{lemma}
\begin{proof}
Let \(g(t)=\left(1+{1}/{t}\right)t^{\frac{1}{2(t+1)}}\), and this follows that
\[
g'(t)=-t^{\frac{1}{2(t+1)}}\left(\frac{1}{2t^2}+\frac{\ln t}{2t(t+1)}\right)<0 \text{ for }t\geq1.
\]
Therefore, \(g(t)\) is a decreasing function, and \(g(t)\leq g(1)=2\) for \(t\geq1\). This completes the proof as desired.
\end{proof}

\begin{lemma}\label{lem:tmp}
For \(k\geq8\) and \(\hat{p}\in[0,1/2]\), we have
\[
(1-2\hat{p})e^{k(\frac{1}{2}-\hat{p})+\frac{k}{2}\log 2\hat{p}}\leq{\sqrt{2}\hat{p}}/{\sqrt{k}}.
\]
\end{lemma}
\begin{proof}
We first write
\[
f(\hat{p})=(1-2\hat{p})e^{k(\frac{1}{2}-\hat{p})+\frac{k}{2}\log 2\hat{p}}/2\hat{p}=(1-2\hat{p})e^{\frac{k}{2}(1-2\hat{p})} (2\hat{p})^{\frac{k}{2}-1}
\]
and the derivative is given by
\[
f'(\hat{p})=\frac{2k}{\hat{p}}\left(\hat{p}^2-\hat{p}+\frac{1}{4}-\frac{1}{2k}\right)e^{\frac{k}{2}(1-2\hat{p})} (2\hat{p})^{\frac{k}{2}-1}.
\]
Solving \(f'(\hat{p})=0\) gives the optimal solution
\[
\hat{p}^*=\frac{1}{2}\left(1-\sqrt{\frac{2}{k}}\right) \quad\text{ for }\quad\hat{p}^*\in[0,1/2].
\]
It is easy to find that
\begin{equation}\label{tmp:t1}
f(\hat{p})\leq \max_{\hat{p}\in[0,1/2]} f(\hat{p})= \max\{f(0),f(1/2),f(\hat{p}^*)\}= f(\hat{p}^*)
\end{equation}
because \(f(\hat{p})\) is continuous for \(\hat{p}\in[0,1/2]\). We further have
\[
f(\hat{p}^*)=\sqrt{\frac{2}{k}}\left(1-\sqrt{\frac{2}{k}}\right)^{\frac{k}{2}-1}\exp\left({\frac{\sqrt{2k}}{2}}\right)=\sqrt{\frac{2}{k}} \exp(g(k))
\]
where
\[
g(k)=\sqrt{\frac{2}{k}}+\left(\frac{k}{2}-1\right)\ln\left(1-\sqrt{\frac{2}{k}}\right)\leq 2\sqrt{\frac{2}{k}}-\sqrt{\frac{k}{2}}\leq -1
\]
where we use the facts \(\ln(1-x)\leq -x\) and \(k\geq8\). Therefore, we have
\[
f(\hat{p}^*)\leq {\sqrt{2}}/{e\sqrt{k}}\leq {\sqrt{2}}/2{\sqrt{k}}
\]
This lemma follows by combining with Eqn.~\eqref{tmp:t1}.
\end{proof}

Based on Lemma~\ref{lem:tmp}, we have
\begin{lemma}\label{lem:key:knn}
For \(k\geq8\), let \(Z=\sum_{i=1}^k Z_i/k\), where \(Z_1, Z_2, \ldots,Z_k\) are independent Bernoulli random variables with parameters \(\hat{p}_1, \hat{p}_2, \ldots, \hat{p}_k\), respectively, i.e., \(Z_i\sim\text{B}(\hat{p}_i)\) for \(i\in[k]\). We set \(\hat{p}=\sum_{i=1}^k \hat{p}_i/k\),  \(p=(\hat{p}-\tau)/(1-2\tau)\), and let Bernoulli random variable \(y\sim \text{B}(p)\). We have
\begin{eqnarray*}
\lefteqn{E_{Z_1, \ldots,Z_k} \Pr\nolimits_{y\sim\text{B}(p)}[y\neq I[Z> 1/2]]}\\
&\leq&\left(1+\sqrt\frac{2}{{k}}\right)\Pr_{y\sim\text{B}(p)}[y\neq I[\hat{p}>1/2]]+ \frac{\sqrt{2}\tau}{\sqrt{k}(1-2\tau)}.
\end{eqnarray*}
\end{lemma}

\begin{proof}
We will present detailed proof for \(\hat{p}\leq 1/2\), and similar consideration could be proceeded for  \(\hat{p}> 1/2\). For \(\hat{p}\leq 1/2\), we have
\[
\Pr_{y\sim\text{B}(p)}[y\neq I[\hat{p}>1/2]]=p
\]
and
\begin{eqnarray*}
\lefteqn{E_{Z_1, \ldots,Z_k} \Pr_{y\sim\text{B}(p)}[y\neq I[Z> 1/2]]=p\Pr[Z\leq1/2]+(1-p)\Pr[Z>1/2]}\\
&=&p(1-\Pr[Z>1/2])+(1-p)\Pr[Z>1/2]=p+(1-2p)\Pr[Z>1/2].
\end{eqnarray*}
Based on the Chernoff's bound, we have
\[
\Pr[Z>1/2]=\Pr[Z-\hat{p}>1/2-\hat{p}]\leq e^{k(\frac{1}{2}-\hat{p})+\frac{k}{2}\log 2\hat{p}}.
\]
For \(k\geq8\), we have
\begin{eqnarray*}
(1-2p)\Pr[Z>1/2]&=& \frac{1-2\hat{p}}{1-2\tau}\Pr[Z>1/2]\\
&\leq& \frac{1-2\hat{p}}{1-2\tau} e^{k(\frac{1}{2}-\hat{p})+\frac{k}{2}\log 2\hat{p}}\leq \frac{\sqrt{2}\hat{p}}{(1-2\tau)\sqrt{k}}
\end{eqnarray*}
where the first equality holds from \(1-2p=(1-2\hat{p})/(1-2\tau)\), and the last inequality holds from Lemma~\ref{lem:tmp}. We complete the proof from the fact \(\hat{p}=p+\tau-2p\tau\).
\end{proof}

\subsection{Proof of Theorem~\ref{thm:nearneighbor}}\label{sec:pf:thm:nearneighbor}
From the definition \(R_\D(h_{\hat{S}_n})=E_{(\x,y)\sim\D}[I[h_{\hat{S}_n}(\x)\neq y]]\), we first observe that \(E_{\hat{S}_n\sim\hD^n}[R_\D(h_{\hat{S}_n})]\) is the probability of training sample \(\hat{S}_n\sim\hD^n\) and \((\x,y)\sim\D\) such that \(\hat{y}_{\pi_1(\x)}\) is different from \(y\).  We have
\begin{eqnarray*}
\lefteqn{E_{\hat{S}_n\sim\hD^n}[R_\D(h_{\hat{S}_n})]=E_{\hat{S}_n\sim\hD^n}[E_{(\x,y)\sim\D}[I[h_{\hat{S}_n}(\x)\neq y]]]}\\
&=&E_{\x,\x_1,\ldots,\x_n\sim\D_\X^{n+1},y\sim\text{B}(\eta(\x)),\hat{y}\sim\text{B}(\hat{\eta}(\x_{\pi_1(\x)}))}[I[\hat{y}\neq y]]\\
&=&E_{\x,\x_1,\ldots,\x_n\sim\D_\X^{n+1}}\left[\Pr_{y\sim\text{B}(\eta(\x)),\hat{y}\sim\text{B}(\hat{\eta}(\x_{\pi_1(\x)}))}[I[\hat{y}\neq y]]\right].
\end{eqnarray*}
where \(\hat{y}\sim\hat{\eta}(\x_{\pi_1(\x)})\) from corrupted distribution \(\hD\). Given any two instances \(\x\) and \(\x'\), we have
\begin{eqnarray*}
\lefteqn{\Pr\nolimits_{y\sim\text{B}(\eta(\x)),\hat{y}'\sim\text{B}(\hat{\eta}(\x'))}[y\neq \hat{y}']}\\
&=&\eta(\x)(1-\hat{\eta}(\x'))+\hat{\eta}(\x')(1-\eta(\x))\\
&=&\eta(\x)+\hat{\eta}(\x')(1-2\eta(\x))\\
&=&\eta(\x)+\eta(\x)(1-2\eta(\x))+(\hat{\eta}(\x')-\eta(\x))(1-2\eta(\x))\\
&=&2\eta(\x)(1-\eta(\x))+(\hat{\eta}(\x')-\eta(\x))(1-2\eta(\x)).
\end{eqnarray*}
For noisy label \(\hat{y}'\), we have
\[
\hat{\eta}(\x')=\eta(\x')(1-\tau)+(1-\eta(\x'))\tau=\eta(\x')(1-2\tau)+\tau,
\]
which implies
\[
\hat{\eta}(\x')-\eta(\x)=(\eta(\x')-\eta(\x))(1-2\tau)+\tau(1-2\eta(\x)).
\]
This follows that
\begin{eqnarray*}
\lefteqn{\Pr\nolimits_{y\sim\text{B}(\eta(\x)),\hat{y}\sim\text{B}(\hat{\eta}(\x_{\pi_1(\x)}))}[y\neq \hat{y}]}\\
&=&\tau+(2-4\tau)\eta(\x)(1-\eta(\x))+(\eta(\x_{\pi_1(\x)})-\eta(\x))(1-2\eta(\x))(1-2\tau).
\end{eqnarray*}
Therefore, we have
\begin{eqnarray}
\lefteqn{E_{\hat{S}_n\sim\hD^n}[R_\D(h_{\hat{S}_n})]=\tau+(1-2\tau)E_{\x\sim\D_\X}[2\eta(\x)(1-\eta(\x))]}\label{eq:1nn:tmp1}\\
&&+E_{\x,\x_1,\ldots,\x_n\sim\D_\X^{n+1}}[(\eta(\x_{\pi_1(\x)})-\eta(\x))(1-2\eta(\x))(1-2\tau)].\label{eq:1nn:tmp2}
\end{eqnarray}
For Eqn.~\eqref{eq:1nn:tmp1}, we have \(\eta(\x)(1-\eta(\x))\leq \min\{\eta(\x),1-\eta(\x)\}\) from \(\eta(\x)\in[0,1]\), and
\begin{eqnarray*}
2\eta(\x)(1-\eta(\x))&=&2\min\{\eta(\x),1-\eta(\x)\}(1-\min\{\eta(\x),1-\eta(\x)\})\\
&=&\min\{\eta(\x),1-\eta(\x)\}(2-2\min\{\eta(\x),1-\eta(\x)\})\\
&\geq&\min\{\eta(\x),1-\eta(\x)\}
\end{eqnarray*}
where the last inequality holds from \(\min\{\eta(\x),1-\eta(\x)\}\leq1/2\). This follows
\begin{equation}\label{eq:1nn:tmp3}
R^*_\D\leq E_{\x\sim\D_\X}[2\eta(\x)(1-\eta(\x))]\leq 2R^*_\D.
\end{equation}
For Eqn.~\eqref{eq:1nn:tmp2}, we have
\begin{eqnarray}
\lefteqn{|E_{\x,\x_1,\ldots,\x_n\sim\D_\X^{n+1}}[(\eta(\x_{\pi_1(\x)})-\eta(\x))(1-2\eta(\x))(1-2\tau)]|}\nonumber\\
&\leq&E_{\x,\x_1,\ldots,\x_n\sim\D_\X^{n+1}}[|(\eta(\x_{\pi_1(\x)})-\eta(\x))(1-2\eta(\x))(1-2\tau)|]\nonumber\\
&\leq&(1-2\tau)LE_{\x,\x_1,\ldots,\x_n\sim\D_\X^{n+1}}[\|\x_{\pi_1(\x)}-\x\|]\nonumber\\
&=&(1-2\tau)LE_{\x,\hat{S}_n}[\|\x_{\pi_1(\x)}-\x\|]\label{eq:1nn:tmp4}
\end{eqnarray}
where the last inequality holds from \(|1-2\eta(\x)|\leq1\) and the \(L\)-Lipschitz assumption of \(\eta(\x)\). This remains to bound \(E_{\x,\hat{S}_n}[\|\x_{\pi_1(\x)}-\x\|]\), and we proceed exactly as in \citep{Shalev-Shwartz:Ben-David2014}. Fixed \(\mu>0\), and let \(C_1,\ldots,C_r\) be the cover of instance space \(\X\) using boxes of length \(\mu\), where \(r=(1/\mu)^d\). We have \(\|\x-\x_{\pi_1(\x)}\|\leq \sqrt{d}\mu\) for \(\x\) and \(\x_{\pi_1(\x)}\) in the same box; otherwise, \(\|\x-\x_{\pi_1(\x)}\|\leq \sqrt{d}\). This follows that
\begin{eqnarray*}
\lefteqn{E_{\x,\hat{S}_n}\left[\|\x_{\pi_1(\x)}-\x\|\right]}\\
&&\leq E_{\hat{S}_n}\Big[\sum_{i=1}^r\Pr[C_i](\sqrt{d}\mu I[\hat{S}_n\cap C_i\neq\emptyset] + \sqrt{d}I[\hat{S}_n\cap C_i=\emptyset])\Big].
\end{eqnarray*}
From the fact that
\[
P(C_i)E_{\hat{S}_n}[I[\hat{S}_n\cap C_i=\emptyset]]=P(C_i)(1-P(C_i))^n\leq 1/ne,
\]
we have
\[
E_{\x,\hat{S}_n}\left[\|\x_{\pi_1(\x)}-\x\|\right]\leq \sqrt{d}(\mu+r/ne)= \sqrt{d}(\mu+1/ne\mu^d)
\]
which implies that, by setting \(\mu=(d/ne)^{1/(d+1)}\) and from Lemma~\ref{lem:1nn},
\[
E_{\x,\hat{S}_n}\left[\|\x_{\pi_1(\x)}-\x\|\right]\leq \sqrt{d}\left(1+\frac{1}{d}\right)\left(\frac{d}{ne}\right)^\frac{1}{d+1}\leq \frac{3\sqrt{d}}{2n^{\frac{1}{1+d}}}.
\]
From Eqn.~\eqref{eq:1nn:tmp4}, we have
\begin{equation*}
|E_{\x,\x_1,\ldots,\x_n\sim\D_\X^{n+1}}[(\eta(\x_{\pi_1(\x)})-\eta(\x))(1-2\eta(\x))(1-2\tau)]|\leq 3\sqrt{d}(1-2\tau)L/2n^{\frac{1}{1+d}}.
\end{equation*}
By combining the above with Eqns.~\eqref{eq:1nn:tmp1}-\eqref{eq:1nn:tmp3}, we complete the proof.\qed

\begin{lemma}\label{lem:1nn}
For integer \(d\geq1\), we have
\[
\left(1+\frac{1}{d}\right)\left(\frac{d}{e}\right)^\frac{1}{d+1}\leq \frac{3}{2}
\]
\end{lemma}
\begin{proof}
Let \(g(d)=(1+{1}/{d})({d}/{e})^\frac{1}{d+1}\). We have
\[
g'(d)=\frac{1-\ln d}{d(1+d)}\left(\frac{d}{e}\right)^\frac{1}{d+1}.
\]
By setting \(g'(d)=0\), we have \(d=e\) and \(g(d)\leq g(e)\leq3/2\). This completes the proof.
\end{proof}

\subsection{Proof of Theorem~\ref{thm:ourRkNN}}\label{sec:pf:thm:ourRkNN}
The proof is similar to that of Theorem~\ref{theorem:noiknn} in Section~\ref{sec:pf:theorem:noiknn}, whereas the boundary of corrupted conditional probability changes from \(1/2\) to \((1+\tau_--\tau_+)/2\). Recall \(\mathcal{E}^b =\{\x\in\X\colon \eta(\x)= 1/2 \}\), and it is necessary to introduce two sets as follows
\begin{eqnarray*}
\mathcal{E}^{r+}_\Delta &=&\{\x\in\X\colon \eta(\x)> 1/2,~~\hat{\eta}(\x)\geq (1+\tau_--\tau_+)/2 +\Delta \}, \\
\mathcal{E}^{r-}_\Delta &=&\{\x\in\X\colon \eta(\x)< 1/2,~~\hat{\eta}(\x)\leq (1+\tau_--\tau_+)/2 -\Delta \},
\end{eqnarray*}
for \(\Delta>0\). We denote by
\begin{eqnarray}
\A^{r}_\Delta&=&\X\setminus (\E^{r+}_\Delta \cup \E^{r-}_\Delta\cup \E^b)\nonumber\\
\A_0^{r}&=& \X\setminus (\E^{r+}_0 \cup \E^{r-}_0 \cup \E^b)\nonumber\\
&=&\{\x\in\X\colon (\eta(\x)-\tfrac{1}{2})(\hat{\eta}(\x)-(1+\tau_--\tau_+)/2)<0\}. \label{eq:Ar0}
\end{eqnarray}
We now present a general theorem for the consistency of the proposed R\(k\)NN algorithm as follows:
\begin{theorem}\label{theorem:rnn}
Let \(\hat{S}_n\) be a corrupted sample with noise proportions \(\tau_-\) and \(\tau_+\).  Let \(h^{rk}_{\hat{S}_n}\) be the output hypothesis of applying our R\(k\)NN algorithm to \(\hat{S}_n\). We have
\begin{eqnarray*}
{E}_{\hat{S}_n\sim\hD^n}[R_\D(h^{rk}_{\hat{S}_n})] &\leq& R_D^*+ \Pr\nolimits_{\x\sim\D_\X}[\x\in\A^r_\Delta]\\
&& +\frac{1}{\sqrt{k}}+2\big((1-\tau_+-\tau_-)L/\sqrt{d}\big)^{\frac{d}{1+d}} \Big(\frac{k}{n}\Big)^{\frac{1}{1+d}}
\end{eqnarray*}
where \(\Delta=\max\{2\sqrt{d}((1-\tau_+-\tau_-)L)^\frac{d}{1+d}(k\sqrt{d}/n)^\frac{1}{1+d}\), \(\sqrt{\log k/ k}\}\).
\end{theorem}

This theorem is similar to Theorem~\ref{theorem:noiknn}, whereas the boundary of corrupted conditional probability \(\hat{\eta}(\x)\) changes from \(1/2\) to \((1+\tau_--\tau_+)/2\) by random noise. Based on Theorem~\ref{theorem:rnn}, we have
\begin{equation*}
{E}_{\hat{S}_n\sim\hD^n}[R_\D(h^{rk}_{\hat{S}_n})] - R_D^*\to \Pr\nolimits_{\x\sim\D_\X}[\x\in\A^r_0]
\end{equation*}
if \(k=k(n)\to \infty\) and \(k/n\to0\) as \(n\to\infty\); we also have
\begin{equation*}
{E}_{\hat{S}_n\sim\hD^n}[R_\D(h^{rk}_{\hat{S}_n})] - R_D^*\to \Pr\nolimits_{\x\sim\D_\X}[\x\in\A^r_0]+1/\sqrt{k}
\end{equation*}
for constant \(k\) as \(n\to\infty\).  From Eqn.~\eqref{eq:Ar0}, we have
\begin{eqnarray*}
\A^r_0&=&\{\x\in\X\colon (\eta(\x)-1/2)(\hat{\eta}(\x)-(1+\tau_--\tau_+)/2)<0\}\\
    &=& \{\x\in\X\colon (1-\tau_--\tau_+)(\eta(\x)-1/2)^2<0\}=\emptyset
\end{eqnarray*}
which implies \(\Pr\nolimits_{\x\sim\D_\X}[\x\in\A^r_0]=0\). This completes the proof of Theorem~\ref{thm:ourRkNN}.

\

\noindent\textit{Proof of Theorem~\ref{theorem:rnn}}
Without loss of generality, we assume \(\tau_-\neq\tau_+\). Based on the total probability theorem, we have
\begin{eqnarray}
\lefteqn{\Pr_{(\x,y)\sim\D}[{h}^{rk}_{\hat{S}_n}(\x)\neq y]}\nonumber\\
&=&\Pr_{\x\sim\D_\X}[h^{rk}_{\hat{S}_n}(\x)\neq y,\x\in \A^r_\Delta] +\Pr_{(\x,y)\sim\D}[h^{rk}_{\hat{S}_n}(\x)\neq y,\x\notin \A^r_\Delta] \nonumber\\
&\leq&\Pr_{\x\sim\D_\X}[\x\in \A^r_\Delta] +\Pr_{(\x,y)\sim\D}[h^{rk}_{\hat{S}_n}(\x)\neq y,\x\notin \A^r_\Delta].\label{eq:tmp0t}
\end{eqnarray}
Fixed \(\mu>0\), let \(C_1,\ldots,C_r\) be the cover of \(\X=[0,1]^d\) with boxes of length \(\mu\), and we have \(r=(1/\mu)^d\). Denote by two random events
\begin{eqnarray*}
\Gamma_1(\x,\x')&=&\{\text{there exists some } C_i\text{ such that } \x\in C_i\text{ and }\x'\in C_i\},\\
\Gamma_2(\x,\x')&=&\{\text{for each }C_i, \text{ we have either }\x\notin C_i\text{ or }\x'\notin C_i\}.
\end{eqnarray*}
By total probability theorem, we have
\begin{eqnarray}
\lefteqn{\Pr_{(\x,y)\sim\D}[h^{rk}_{\hat{S}_n}(\x)\neq y,\x\notin \A^r_\Delta]}\nonumber\\
&=&\Pr_{(\x,y)\sim\D}[h^{rk}_{{\hat{S}}_n}(\x)\neq y, \x\notin\A^r_\Delta\big|\Gamma_2(\x,\x_{\pi_{k}(\x)})] \Pr_{(\x,y)\sim\D}[\Gamma_2(\x,\x_{\pi_{k}(\x)})]\nonumber\\
&&+\Pr_{(\x,y)\sim\D}[h^{rk}_{{\hat{S}}_n}(\x)\neq y, \x\notin\A^r_\Delta\big|\Gamma_1(\x,\x_{\pi_{k}(\x)})] \Pr_{(\x,y)\sim\D}[\Gamma_1(\x,\x_{\pi_{k}(\x)})]\nonumber\\
&\leq& \Pr_{(\x,y)\sim\D}[\Gamma_2(\x,\x_{\pi_{k}(\x)})]
+\Pr_{(\x,y)\sim\D}\big[h^{rk}_{{\hat{S}}_n}(\x)\neq y, \x\notin\A^r_\Delta\big|\Gamma_1(\x,\x_{\pi_{k}(\x)}) \big].\label{eq:tmp2t}
\end{eqnarray}
From Lemma~\ref{lem:knn:tmp1}, the first term in the above can be upper bounded by
\begin{equation}\label{eq:tmp333}
E_{\hat{S}_n\sim\hD^n}\big[\Pr_{(\x,y)\sim\D}[\Gamma_2(\x,\x_{\pi_{k}(\x)})]\big]\leq {2rk}/{n}.
\end{equation}
For the second term, we fix \(\x_1,\x_2,\ldots,\x_n\) and \(\x\notin \A^r_\Delta\). Assume that \(\x_1,\ldots,\x_k\) are \(k\)-nearest neighbors of \(\x\). Let \(\hat{\eta}(\x_1),\ldots,\hat{\eta}(\x_k)\) be conditional probabilities  w.r.t. corrupted distribution \(\hD\), and set \(\hat{p}=\sum_{i=1}^k \hat{\eta}(\x_i)/k\).

\noindent If \(\x\in\E^b\), then we have
\[
\Pr_{y\sim\text{B}(\eta(\x))}[h^{rk}_{{\hat{S}}_n}(\x)\neq y]=1/2=\min\{\eta(\x),1-\eta(\x)\}.
\]
If \(\x\in\E^{r-}_\Delta\), then we set \(\Delta=2(1-\tau_+-\tau_-)L\mu\sqrt{d}\), and have
\begin{equation}\label{eq:tmp3t}
\hat{p}<(1+\tau_--\tau_+)/2-\Delta/2
\end{equation}
from \(|\hat{\eta}(\x)-\hat{p}|\leq \Delta/2\) by Eqn.~\eqref{lem:relation}. We also have
\begin{eqnarray*}
\lefteqn{\Pr_{y\sim\text{B}(\eta(\x))}[h^{rk}_{{\hat{S}}_n}(\x)\neq y] }\\
&=&\eta(\x)I[h^{rk}_{{\hat{S}}_n}(\x)=0]+(1-\eta(\x))I[h^{rk}_{{\hat{S}}_n}(\x)=1]\\
&=&\eta(\x)+(1-2\eta(\x))I[h^{rk}_{{\hat{S}}_n}(\x)=1].
\end{eqnarray*}
For our R\(k\)NN algorithm, it holds that, for \(\tau_->\tau_+\)
\[
I[h^{rk}_{{\hat{S}}_n}(\x)=1] =I\Big[\sum_{i=1}^k\frac{\hat{y_i}}{k}\geq\frac{1}{2}\Big] - I\Big[\frac{1}{2}<\sum_{i=1}^k\frac{\hat{y_i}}{k}\leq\frac{1+\tau_--\tau_+}{2}\Big];
\]
and for \(\tau_-<\tau_+\)
\[
I[h^{rk}_{{\hat{S}}_n}(\x)=1]= I\Big[\sum_{i=1}^k\frac{\hat{y_i}}{k}\geq\frac{1}{2}\Big] + I\Big[\frac{1+\tau_--\tau_+}{2} \leq\sum_{i=1}^k\frac{\hat{y_i}}{k}<\frac{1}{2}\Big].
\]
This implies that
\[
I[h^{rk}_{{\hat{S}}_n}(\x)=1]\leq I\Big[\sum_{i=1}^k\frac{\hat{y_i}}{k}\geq \frac{1+\tau_--\tau_+}{2}\Big].
\]
Combining with Eqn.~\eqref{eq:tmp3t} and Chernoff's bounds, we have
\[
\Pr_{\hat{y}_1\sim\text{B}(\hat{\eta}(\x_1)),\ldots, \hat{y}_k\sim\text{B}(\hat{\eta}(\x_k))}\Big[\sum_{i=1}^k\frac{\hat{y_i}}{k} -\hat{p}\geq\frac{1+\tau_--\tau_+}{2}-\hat{p}\Big]\leq \exp(-k\Delta^2/2).
\]
This follows that
\begin{multline}
{E}_{\hat{y}_1\sim\text{B}(\hat{\eta}(\x_1)),\ldots, \hat{y}_k\sim\text{B}(\hat{\eta}(\x_k))} \Big[\Pr\nolimits_{y\sim\text{B}(\eta(\x))}[I[h^{rk}_{{\hat{S}}_n}(\x)\neq y]]\Big]\\
\leq \eta(\x)+ \exp(-k\Delta^2/2),~~~~~~~~
\end{multline}
and it is noteworthy of \(\eta(\x)=\min\{\eta(\x),1-\eta(\x)\}\) for \(\x\in\E_\Delta^-\).

Similarly, we prove that, for \(\x\in\E_\Delta^+\),
\begin{multline}\label{eq:pftt}
{E}_{\hat{y}_1\sim\text{B}(\hat{\eta}(\x_1)),\ldots, \hat{y}_k\sim\text{B}(\hat{\eta}(\x_k))} \Big[\Pr\nolimits_{y\sim\text{B}(\eta(\x))}[I[h^{rk}_{{\hat{S}}_n}(\x)\neq y]]\Big]\\
\leq 1-\eta(\x)+ \exp(-k\Delta^2/2),~~~~~~~~
\end{multline}
and  it is noteworthy of \(1-\eta(\x)=\min\{\eta(\x),1-\eta(\x)\}\) for \(\x\in\E_\Delta^+\).

Combining with Eqns.~\eqref{eq:tmp0t}-\eqref{eq:pftt}, we have
\begin{eqnarray*}
{E}_{\hat{S}_n\sim\hD^n}[R(h^{rk}_{\hat{S}_n})]&=&{E}_{\hat{S}_n\sim\hD^n}[\Pr_{(\x,y)\sim\D}[h^{rk}_{\hat{S}_n}(\x)\neq y] \\
&\leq&  \Pr_{\x\sim\D_\X}[\x\in \A^r_\Delta]+\exp(-k\Delta^2/2)+R^*_\D+{2rk}/{n}.
\end{eqnarray*}
where \(\Delta=2(1-\tau_+-\tau_-)L\mu\sqrt{d}\). From \(r=1/\mu^d\), we set
\[
\mu=\max\Big\{\frac{\sqrt{\log k/k}}{2(1-\tau_+-\tau_-)L\sqrt{d}},\Big(\frac{k\sqrt{d}}{n(1-\tau_+-\tau_-)L}\Big)^{\frac{1}{d+1}}\Big\}
\]
which completes the proof by simple calculations.\qed

\section{Experiments}\label{sec:exp}
This section verifies theoretical results on synthetic dataset in Section~\ref{sec:exp:syn}, and shows the effectiveness of R\(k\)NN on benchmark datasets in Section~\ref{sec:exp:ben}, followed by parameter analysis in Section~\ref{sec:exp:para}.

\subsection{Synthetic Dataset}\label{sec:exp:syn}
We consider the instance space \(\X=[0,1]^2\), which is similar to synthetic dataset in \citep{Berlind:Urner2015}. Let \(\D_\X\) be a uniform distribution over \(\X\), and  \(\eta(x_1, x_2) = (1 - \sin(2\pi x_1)\sin(2\pi x_2))/2\). We select \(8000\) and \(7000\) examples (i.i.d) for training and testing, respectively. Given noise proportions \((\tau_+,\tau_-)\), the labels of training data are flipped accordingly and independently for \(20\) times with different random seeds, and the average classification error is calculated on test data without noise corruptions.

Figure~\ref{fig:syn}(a) shows that, for asymmetric noises, test error of \(k\)-nearest neighbor does not converge to Bayes error as \(k\) increases, which is nicely in agreement with Theorem~\ref{theorem:noiknn}. Figure~\ref{fig:syn}(b) shows the consistency of \(k\)-nearest neighbor for symmetric noises and large \(k\), as expected in Theorem~\ref{thm:uni-noise}. Figure~\ref{fig:syn}(c) shows the inconsistency of \(1\)-nearest neighbor for symmetric noises as the sample size increases, which verifies Theorem~\ref{thm:nearneighbor} empirically.  Figure~\ref{fig:syn}(d) shows the consistency of our R\(k\)NN approach for asymmetric noises, which presents good supports to Theorem~\ref{thm:ourRkNN}.

\begin{figure*}
\centering
\includegraphics[width=5.4in]{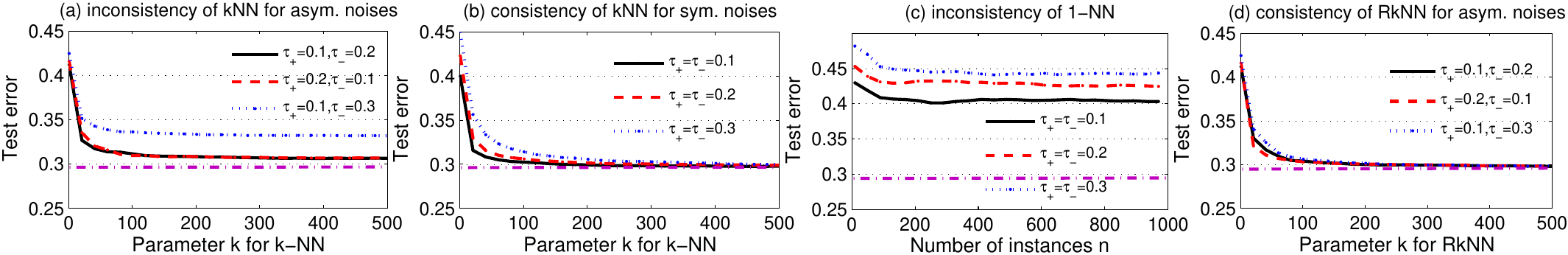}\vspace{-0.1in}
\caption{(a) Convergence of \(k\)-nearest neighbor for asymmetric noises. (b) Convergence of \(k\)-nearest neighbor for symmetric noises. (c) Convergence of \(1\)-nearest neighbor for symmetric noises. (d)~Convergence of our R\(k\)NN for asymmetric noises.
}\label{fig:syn}
\end{figure*}

\subsection{Comparisons on Bechmark Datasets}\label{sec:exp:ben}
We present empirical studies on twenty benchmark datasets\footnote{http://www.ics.uci.edu/{\textasciitilde}mlearn/MLRepository.html}, and the details are summarized in Table~\ref{tab:data}. Most datasets have been used for learning with noisy labels, and the features have been scaled to \([-1,1]\) for all datasets. Multi-class datasets have been transformed into binary ones by randomly partitioning classes into two groups, where each group contains the same cardinality of classes.  We consider three groups of true noise proportions, that is, \((\tau_-,\tau_+)\in\{(0.1,0.2),(0.3,0.1), (0.4,0.4)\}\), and training labels are flipped accordingly with different random seeds.

We evaluate the performance of our R\(k\)NN approach with traditional \(k\)-nearest neighbor \textsf{\(k\)NN}, as well as six state-of-the-art approaches on learning with noisy labels as follows.
\begin{itemize}
  \item \textsf{IR-KSVM}: An importance-reweighting algorithm by kernel hinge-loss method \citep{Liu:Tao2016};
  \item \textsf{IR-LLog}: An importance-reweighting algorithm by linear logistic-loss method \citep{Liu:Tao2016}\footnote{The codes of IR-KSVM and IR-LLog are taken from http://tongliangliu.esy.es};
  \item \textsf{LD-KSVM}: A label-dependent algorithm by kernel hinge-loss method \citep{Natarajan:Dhillon:Ravikumar:Tewari2013};
  \item \textsf{UE-LLog}: An unbiased-estimator algorithm by linear logistic-loss method \citep{Natarajan:Dhillon:Ravikumar:Tewari2013};
  \item \textsf{AROW}: An adaptive regularization of weights \citep{Crammer:Kulesza:Dredze2009};
  \item \textsf{NHERD}: A normal (Gaussian) herd algorithm \citep{Crammer:Lee2010}.
\end{itemize}

\begin{table}
\caption{Benchmark datasets}\label{tab:data}\vspace{0.05in}
\centering
\begin{tabular}{|c|cc||c|cc|}
\hline
\scriptsize datasets &\scriptsize\#inst &\scriptsize\#feat  &\scriptsize datasets &\scriptsize\#inst &\scriptsize\#feat\\
\hline
\scriptsize\textsf{heart}          &\scriptsize270    &\scriptsize13  &\scriptsize\textsf{segment}   &\scriptsize2,310   &\scriptsize19\\
\scriptsize\textsf{ionosphere}     &\scriptsize351    &\scriptsize34  &\scriptsize\textsf{landsat}   &\scriptsize6,435   &\scriptsize36\\
\scriptsize\textsf{housing}        &\scriptsize506    &\scriptsize13  &\scriptsize\textsf{mushroom}  &\scriptsize8,124   &\scriptsize112  \\
\scriptsize\textsf{cancer}         &\scriptsize683    &\scriptsize10  &\scriptsize\textsf{usps}      &\scriptsize9,298   &\scriptsize256 \\
\scriptsize\textsf{diabetes}       &\scriptsize768    &\scriptsize8   &\scriptsize\textsf{pendigits} &\scriptsize10,992  &\scriptsize16 \\
\scriptsize\textsf{vehicle}        &\scriptsize846    &\scriptsize18  &\scriptsize\textsf{letter}    &\scriptsize15,000  &\scriptsize16\\
\scriptsize\textsf{fourclass}      &\scriptsize862    &\scriptsize2   &\scriptsize\textsf{magic04}   &\scriptsize19,020  &\scriptsize10\\
\scriptsize\textsf{german}         &\scriptsize1,000  &\scriptsize24  &\scriptsize\textsf{w8a}       &\scriptsize49,749  &\scriptsize300\\
\scriptsize\textsf{splice}         &\scriptsize1,000  &\scriptsize60  &\scriptsize\textsf{shuttle}   &\scriptsize58,000  &\scriptsize9\\
\scriptsize\textsf{optdigits}      &\scriptsize1,143  &\scriptsize42  &\scriptsize\textsf{acoustic}  &\scriptsize78,823  &\scriptsize50\\
\hline
\end{tabular}
\end{table}

For our R\(k\)NN approach, four-fold cross validation is executed to select predictive parameter \(k\in[5:5:100]\) and noise parameter \(k'\in[5:5:100]\). For \textsf{IR-KSVM} and \textsf{IR-LLog}, we take the default parameters as in \citep{Liu:Tao2016}. For \textsf{LD-KSVM}, we adopt the Gaussion kernels with best width trained by traditional SVM on noise-free data, as introduced in \cite{Natarajan:Dhillon:Ravikumar:Tewari2013}. For \textsf{UE-LLog}, \textsf{AROW} and \textsf{NHERD}, four-fold cross validation is also executed for parameter selections.

{\renewcommand{\arraystretch}{0.86}
\begin{table*}
\caption{Comparison of test accuracy (mean\(\pm\)std.) on benchmark datasets. \(\bullet\)/\(\circ\) indicates that R\(k\)NN is significantly better/worse than the corresponding method (frequent pairwise \(t\)-tests at \(95\%\) significance level).}\label{tab:result}\vspace{0.05in}
\centering\scalebox{0.66}{
\begin{tabular}{|c|c|c|c|c|c|c|c|c|c|}
\hline
\scriptsize datasets &\scriptsize \((\tau_{+}, \tau_{-})\) &\scriptsize Our R\(k\)NN &\scriptsize IR-KSVM &\scriptsize IR-LLog &\scriptsize LD-KSVM &\scriptsize UE-LLog  &\scriptsize AROW  &\scriptsize NHERD &\scriptsize  \(k\)NN\\
\hline

\multirow{3}{*}{\scriptsize\textsf{heart}}
&\scriptsize(0.1, 0.2)  &\scriptsize.8544\(\pm\).0452 &\scriptsize.7941\(\pm\).0318\(\bullet\) &\scriptsize.7088\(\pm\).1302\(\bullet\) &\scriptsize.8000\(\pm\).0362\(\bullet\) &\scriptsize.8029\(\pm\).0533\(\bullet\) &\scriptsize.7721\(\pm\).0451\(\bullet\) &\scriptsize.7721\(\pm\).0525\(\bullet\) &\scriptsize .8353\(\pm\).0458 \(\bullet\)\\
&\scriptsize(0.3, 0.1) &\scriptsize.8706\(\pm\).0403 &\scriptsize.8279\(\pm\).0505\(\bullet\) &\scriptsize.6853\(\pm\).1395\(\bullet\) &\scriptsize.8265\(\pm\).0474\(\bullet\) &\scriptsize.8088\(\pm\).0500\(\bullet\) &\scriptsize.7456\(\pm\).0654\(\bullet\) &\scriptsize.7338\(\pm\).0954\(\bullet\) &\scriptsize .8029\(\pm\).0267 \(\bullet\)\\
&\scriptsize(0.4, 0.4) &\scriptsize.7471\(\pm\).0706 &\scriptsize.5515\(\pm\).1299\(\bullet\) &\scriptsize.6471\(\pm\).1226\(\bullet\) &\scriptsize.6368\(\pm\).1304\(\bullet\) &\scriptsize.6735\(\pm\).0917\(\bullet\) &\scriptsize.6750\(\pm\).0691\(\bullet\) &\scriptsize.6074\(\pm\).1397\(\bullet\) &\scriptsize .7000\(\pm\).0305 \(\bullet\)\\
\hline

\multirow{3}{*}{\scriptsize\textsf{ionosphere}}
&\scriptsize(0.1, 0.2) &\scriptsize.8818\(\pm\).0229 &\scriptsize.8966\(\pm\).0281\(\circ\) &\scriptsize.8205\(\pm\).0363\(\bullet\) &\scriptsize.8875\(\pm\).0323 &\scriptsize.8091\(\pm\).0374\(\bullet\) &\scriptsize.8227\(\pm\).0409\(\bullet\) &\scriptsize.7670\(\pm\).0611\(\bullet\) &\scriptsize .8318\(\pm\).0450\(\bullet\)\\
&\scriptsize(0.3, 0.1)  &\scriptsize.8705\(\pm\).0289 &\scriptsize.8795\(\pm\).0216 &\scriptsize.8284\(\pm\).0353\(\bullet\) &\scriptsize.8841\(\pm\).0232\(\circ\) &\scriptsize.8045\(\pm\).0404\(\bullet\) &\scriptsize.7818\(\pm\).0386\(\bullet\) &\scriptsize.7341\(\pm\).1170\(\bullet\) &\scriptsize .8545\(\pm\).0271\(\bullet\)\\
&\scriptsize(0.4, 0.4) &\scriptsize.7705\(\pm\).0730 &\scriptsize.6727\(\pm\).1025\(\bullet\) &\scriptsize.6989\(\pm\).1025\(\bullet\) &\scriptsize.7341\(\pm\).1137\(\bullet\) &\scriptsize.6727\(\pm\).0923\(\bullet\) &\scriptsize.7102\(\pm\).0981\(\bullet\) &\scriptsize.6227\(\pm\).1653\(\bullet\) &\scriptsize .7932\(\pm\).0246\(\circ\)\\
\hline

\multirow{3}{*}{\scriptsize\textsf{housing}}
&\scriptsize(0.1, 0.2) &\scriptsize.8664\(\pm\).0181 &\scriptsize.8661\(\pm\).0246 &\scriptsize.8701\(\pm\).0145 &\scriptsize.8780\(\pm\).0179\(\circ\) &\scriptsize.8677\(\pm\).0257&\scriptsize.8701\(\pm\).0201 &\scriptsize.8622\(\pm\).0197 &\scriptsize .8513\(\pm\).0223\\
&\scriptsize(0.3, 0.1) &\scriptsize.8693\(\pm\).0250 &\scriptsize.8583\(\pm\).0445 &\scriptsize.8693\(\pm\).0433 &\scriptsize.8677\(\pm\).0356 &\scriptsize.8654\(\pm\).0357 &\scriptsize.8751\(\pm\).0355 &\scriptsize.8614\(\pm\).0347 &\scriptsize .8409\(\pm\).0151\(\bullet\)\\
&\scriptsize(0.4, 0.4) &\scriptsize.8157\(\pm\).0428 &\scriptsize.7756\(\pm\).0476\(\bullet\) &\scriptsize.7874\(\pm\).0609\(\bullet\) &\scriptsize.7173\(\pm\).0687\(\bullet\) &\scriptsize.7976\(\pm\).0393 &\scriptsize.7787\(\pm\).0489\(\bullet\) &\scriptsize.7063\(\pm\).1412\(\bullet\) &\scriptsize .8085\(\pm\).0631\\
\hline

\multirow{3}{*}{\scriptsize\textsf{cancer}}
&\scriptsize(0.1, 0.2) &\scriptsize.9731\(\pm\).0114 &\scriptsize.9673\(\pm\).0133 &\scriptsize.9661\(\pm\).0132 &\scriptsize.9690\(\pm\).0126 &\scriptsize.9567\(\pm\).0146\(\bullet\)  &\scriptsize.9696\(\pm\).0134 &\scriptsize.9690\(\pm\).0110 &\scriptsize .9754\(\pm\).0076\\
&\scriptsize(0.3, 0.1) &\scriptsize.9760\(\pm\).0125 &\scriptsize.9661\(\pm\).0120 &\scriptsize.9561\(\pm\).0181\(\bullet\)  &\scriptsize.9655\(\pm\).0164 &\scriptsize.9503\(\pm\).0223\(\bullet\)  &\scriptsize.9345\(\pm\).0308\(\bullet\)  &\scriptsize.9444\(\pm\).0362\(\bullet\) &\scriptsize .9719\(\pm\).0151\\
&\scriptsize(0.4, 0.4) &\scriptsize.9006\(\pm\).1031 &\scriptsize.9345\(\pm\).0324\(\circ\) &\scriptsize.8953\(\pm\).0370 &\scriptsize.8819\(\pm\).0642\(\bullet\)  &\scriptsize.8725\(\pm\).0670\(\bullet\)  &\scriptsize.9072\(\pm\).0425 &\scriptsize.8830\(\pm\).0499\(\bullet\) &\scriptsize .9135\(\pm\).0877\(\circ\)\\
\hline

\multirow{3}{*}{\scriptsize\textsf{diabetes}}
&\scriptsize(0.1, 0.2)  &\scriptsize.7531\(\pm\).0276 &\scriptsize.7641\(\pm\).0382\(\circ\) &\scriptsize.7464\(\pm\).0379 &\scriptsize.7651\(\pm\).0258\(\circ\) &\scriptsize.7578\(\pm\).0368 &\scriptsize.7603\(\pm\).0227 &\scriptsize.7500\(\pm\).0256 &\scriptsize .7354\(\pm\).0203\(\bullet\)\\
&\scriptsize(0.3, 0.1) &\scriptsize.7429\(\pm\).0361 &\scriptsize.7255\(\pm\).0430\(\bullet\) &\scriptsize.7307\(\pm\).0449 &\scriptsize.7448\(\pm\).0358 &\scriptsize.7505\(\pm\).0332 &\scriptsize.7115\(\pm\).0342\(\bullet\) &\scriptsize.7376\(\pm\).0434 &\scriptsize .7250\(\pm\).0419\(\bullet\)\\
&\scriptsize(0.4, 0.4) &\scriptsize.6923\(\pm\).0659 &\scriptsize.6411\(\pm\).1044\(\bullet\) &\scriptsize.6771\(\pm\).0641 &\scriptsize.6682\(\pm\).0905\(\bullet\) &\scriptsize.6807\(\pm\).0874 &\scriptsize.6914\(\pm\).0769 &\scriptsize.7089\(\pm\).0403 &\scriptsize .6896\(\pm\).0709\\
\hline

\multirow{3}{*}{\scriptsize\textsf{vehicle}}
&\scriptsize(0.1, 0.2) &\scriptsize.9615\(\pm\).0211 &\scriptsize.9624\(\pm\).0186 &\scriptsize.8725\(\pm\).1157\(\bullet\) &\scriptsize.9615\(\pm\).0177 &\scriptsize.9330\(\pm\).0382\(\bullet\) &\scriptsize.9459\(\pm\).0319\(\bullet\) &\scriptsize.8734\(\pm\).1009\(\bullet\) &\scriptsize .9450\(\pm\).0183\(\bullet\)\\
&\scriptsize(0.3, 0.1) &\scriptsize.9505\(\pm\).0230 &\scriptsize.9275\(\pm\).0285\(\bullet\) &\scriptsize.6789\(\pm\).2000\(\bullet\) &\scriptsize.9284\(\pm\).0255\(\bullet\) &\scriptsize.9028\(\pm\).0426\(\bullet\) &\scriptsize.9064\(\pm\).0538\(\bullet\) &\scriptsize.8275\(\pm\).0915\(\bullet\) &\scriptsize .9468\(\pm\).0364\\
&\scriptsize(0.4, 0.4) &\scriptsize.8394\(\pm\).0514 &\scriptsize.7864\(\pm\).1495\(\bullet\) &\scriptsize.6358\(\pm\).1104\(\bullet\) &\scriptsize.7908\(\pm\).1154\(\bullet\) &\scriptsize.7523\(\pm\).0852\(\bullet\) &\scriptsize.8202\(\pm\).0729\(\bullet\) &\scriptsize.7743\(\pm\).1212\(\bullet\) &\scriptsize .8037\(\pm\).0724\(\bullet\)\\
\hline

\multirow{3}{*}{\scriptsize\textsf{fourclass}}
&\scriptsize(0.1, 0.2) &\scriptsize.9968\(\pm\).0038 &\scriptsize.8130\(\pm\).0236\(\bullet\) &\scriptsize.7528\(\pm\).0273\(\bullet\) &\scriptsize.8074\(\pm\).0303\(\bullet\) &\scriptsize.7583\(\pm\).0281\(\bullet\) &\scriptsize.6926\(\pm\).0293\(\bullet\) &\scriptsize.7204\(\pm\).0169\(\bullet\) &\scriptsize .9907\(\pm\).0046 \\
&\scriptsize(0.3, 0.1) &\scriptsize.9977\(\pm\).0024 &\scriptsize.8167\(\pm\).0304\(\bullet\) &\scriptsize.7597\(\pm\).0282\(\bullet\) &\scriptsize.8194\(\pm\).0298\(\bullet\) &\scriptsize.7639\(\pm\).0261\(\bullet\) &\scriptsize.6912\(\pm\).0259\(\bullet\) &\scriptsize.7144\(\pm\).0461\(\bullet\) &\scriptsize .9887\(\pm\).0065\(\bullet\) \\
&\scriptsize(0.4, 0.4) &\scriptsize.8194\(\pm\).0556 &\scriptsize.7093\(\pm\).0580\(\bullet\) &\scriptsize.6769\(\pm\).1031\(\bullet\) &\scriptsize.7347\(\pm\).0755\(\bullet\) &\scriptsize.7361\(\pm\).0468\(\bullet\) &\scriptsize.7056\(\pm\).0285\(\bullet\) &\scriptsize.7139\(\pm\).0317\(\bullet\) &\scriptsize .8128\(\pm\).0589 \\
\hline

\multirow{3}{*}{\scriptsize\textsf{german}}
&\scriptsize(0.1, 0.2) &\scriptsize.7552\(\pm\).0197 &\scriptsize.6692\(\pm\).0259\(\bullet\) &\scriptsize.7603\(\pm\).0259 &\scriptsize.6836\(\pm\).0299\(\bullet\) &\scriptsize.7584\(\pm\).0267 &\scriptsize.6932\(\pm\).0302\(\bullet\) &\scriptsize.6644\(\pm\).0234\(\bullet\) &\scriptsize .7592\(\pm\).0290\\
&\scriptsize(0.3, 0.1)&\scriptsize.7460\(\pm\).0300 &\scriptsize.6932\(\pm\).0278\(\bullet\) &\scriptsize.7340\(\pm\).0309 &\scriptsize.7056\(\pm\).0375\(\bullet\) &\scriptsize.7388\(\pm\).0358 &\scriptsize.6784\(\pm\).0279\(\bullet\) &\scriptsize.6424\(\pm\).0412\(\bullet\) &\scriptsize .7040\(\pm\).0311\(\bullet\)\\
&\scriptsize(0.4, 0.4) &\scriptsize.6704\(\pm\).0523 &\scriptsize.6100\(\pm\).0326\(\bullet\) &\scriptsize.6696\(\pm\).0302 &\scriptsize.5752\(\pm\).0330\(\bullet\) &\scriptsize.6648\(\pm\).0315 &\scriptsize.6128\(\pm\).0415\(\bullet\) &\scriptsize.5824\(\pm\).0563\(\bullet\) &\scriptsize .6944\(\pm\).0122\(\circ\)\\
\hline

\multirow{3}{*}{\scriptsize\textsf{splice}}
&\scriptsize(0.1, 0.2)&\scriptsize.7812\(\pm\).0357 &\scriptsize.7812\(\pm\).0273 &\scriptsize.7612\(\pm\).0388\(\bullet\) &\scriptsize.8080\(\pm\).0318\(\circ\) &\scriptsize.7592\(\pm\).0221\(\bullet\) &\scriptsize.6812\(\pm\).0248\(\bullet\) &\scriptsize.6648\(\pm\).0230\(\bullet\) &\scriptsize .7638\(\pm\).0276\(\bullet\) \\
&\scriptsize(0.3, 0.1) &\scriptsize.7720\(\pm\).0240 &\scriptsize.7384\(\pm\).0307\(\bullet\) &\scriptsize.7524\(\pm\).0355\(\bullet\) &\scriptsize.7764\(\pm\).0315 &\scriptsize.7448\(\pm\).0404\(\bullet\) &\scriptsize.7000\(\pm\).0362\(\bullet\) &\scriptsize.6840\(\pm\).0293\(\bullet\) &\scriptsize .6488\(\pm\).0296\(\bullet\)\\
&\scriptsize(0.4, 0.4)&\scriptsize.6708\(\pm\).0270 &\scriptsize.6180\(\pm\).0287\(\bullet\) &\scriptsize.6428\(\pm\).0318\(\bullet\) &\scriptsize.6180\(\pm\).0271\(\bullet\) &\scriptsize.6288\(\pm\).0460\(\bullet\) &\scriptsize.6012\(\pm\).0356\(\bullet\) &\scriptsize.5740\(\pm\).0392\(\bullet\) &\scriptsize .6587\(\pm\).0344\\
\hline

\multirow{3}{*}{\scriptsize\textsf{optdigits}}
&\scriptsize(0.1, 0.2) &\scriptsize.9990\(\pm\).0017 &\scriptsize.9969\(\pm\).0020 &\scriptsize.9899\(\pm\).0058\(\bullet\) &\scriptsize.9969\(\pm\).0026 &\scriptsize.9612\(\pm\).0121\(\bullet\) &\scriptsize.9969\(\pm\).0020 &\scriptsize.9871\(\pm\).0168\(\bullet\) &\scriptsize .9993\(\pm\).0016\\
&\scriptsize(0.3, 0.1) &\scriptsize.9958\(\pm\).0022 &\scriptsize.9948\(\pm\).0081 &\scriptsize.9720\(\pm\).0140\(\bullet\) &\scriptsize.9969\(\pm\).0035 &\scriptsize.9483\(\pm\).0191\(\bullet\) &\scriptsize.9657\(\pm\).0146\(\bullet\) &\scriptsize.9476\(\pm\).0391\(\bullet\) &\scriptsize .9916\(\pm\).0019\(\bullet\)\\
&\scriptsize(0.4, 0.4) &\scriptsize.9745\(\pm\).0269 &\scriptsize.9587\(\pm\).0461\(\bullet\) &\scriptsize.8084\(\pm\).1398\(\bullet\) &\scriptsize.9682\(\pm\).0204\(\bullet\) &\scriptsize.8517\(\pm\).0502\(\bullet\) &\scriptsize.9269\(\pm\).0402\(\bullet\) &\scriptsize.7892\(\pm\).1281\(\bullet\) &\scriptsize .9685\(\pm\).0231\\
\hline

\multirow{3}{*}{\scriptsize\textsf{segment}}
&\scriptsize(0.1, 0.2) &\scriptsize.8690\(\pm\).0109 &\scriptsize.8649\(\pm\).0136 &\scriptsize.7543\(\pm\).0150\(\bullet\)  &\scriptsize.8626\(\pm\).0160 &\scriptsize.7576\(\pm\).0171\(\bullet\)  &\scriptsize.7604\(\pm\).0178\(\bullet\)  &\scriptsize.7159\(\pm\).0555\(\bullet\)  &\scriptsize .8654\(\pm\).0093\\
&\scriptsize(0.3, 0.1) &\scriptsize.8663\(\pm\).0108 &\scriptsize.8600\(\pm\).0191 &\scriptsize.7356\(\pm\).0143\(\bullet\)  &\scriptsize.8561\(\pm\).0173\(\bullet\)  &\scriptsize.7526\(\pm\).0225\(\bullet\)  &\scriptsize.7543\(\pm\).0189\(\bullet\)  &\scriptsize.7104\(\pm\).0556\(\bullet\) &\scriptsize .8526\(\pm\).0105\(\bullet\) \\
&\scriptsize(0.4, 0.4) &\scriptsize.8123\(\pm\).0269 &\scriptsize.7469\(\pm\).0493\(\bullet\)  &\scriptsize.7057\(\pm\).0192\(\bullet\)  &\scriptsize.7804\(\pm\).0271\(\bullet\)  &\scriptsize.7067\(\pm\).0157\(\bullet\)  &\scriptsize.7145\(\pm\).0269\(\bullet\)  &\scriptsize.6249\(\pm\).0752\(\bullet\) &\scriptsize .7941\(\pm\).0268 \\
\hline

\multirow{3}{*}{\scriptsize\textsf{landsat}}
&\scriptsize(0.1, 0.2) &\scriptsize.9213\(\pm\).0074 &\scriptsize.9183\(\pm\).0076 &\scriptsize.8656\(\pm\).0159\(\bullet\) &\scriptsize.9208\(\pm\).0038 &\scriptsize.8711\(\pm\).0119\(\bullet\) &\scriptsize.8485\(\pm\).0129\(\bullet\) &\scriptsize.8210\(\pm\).0382\(\bullet\) &\scriptsize .9231\(\pm\).0066\\
&\scriptsize(0.3, 0.1)&\scriptsize.9134\(\pm\).0070 &\scriptsize.9119\(\pm\).0139 &\scriptsize.8340\(\pm\).0159\(\bullet\) &\scriptsize.9149\(\pm\).0059 &\scriptsize.8683\(\pm\).0078\(\bullet\) &\scriptsize.8428\(\pm\).0099\(\bullet\) &\scriptsize.7798\(\pm\).0569\(\bullet\) &\scriptsize .9075\(\pm\).0105\(\bullet\)\\
&\scriptsize(0.4, 0.4) &\scriptsize.8701\(\pm\).0132 &\scriptsize.6608\(\pm\).1223\(\bullet\) &\scriptsize.7342\(\pm\).0633\(\bullet\) &\scriptsize.8738\(\pm\).0099 &\scriptsize.7937\(\pm\).0243\(\bullet\) &\scriptsize.8112\(\pm\).0137\(\bullet\) &\scriptsize.6291\(\pm\).1047\(\bullet\) &\scriptsize .8680\(\pm\).0091\\
\hline

\multirow{3}{*}{\scriptsize\textsf{mushroom}}
&\scriptsize(0.1, 0.2)&\scriptsize.9985\(\pm\).0012 &\scriptsize.9975\(\pm\).0023 &\scriptsize.9980\(\pm\).0018 &\scriptsize.9983\(\pm\).0019 &\scriptsize.9900\(\pm\).0040 &\scriptsize.9975\(\pm\).0017 &\scriptsize.9923\(\pm\).0077 &\scriptsize .9987\(\pm\).0014 \\
&\scriptsize(0.3, 0.1) &\scriptsize.9982\(\pm\).0013 &\scriptsize.9922\(\pm\).0079 &\scriptsize.9976\(\pm\).0029 &\scriptsize.9981\(\pm\).0015 &\scriptsize.9891\(\pm\).0072\(\bullet\) &\scriptsize.9976\(\pm\).0018 &\scriptsize.9839\(\pm\).0192\(\bullet\) &\scriptsize .9969\(\pm\).0021 \\
&\scriptsize(0.4, 0.4) &\scriptsize.9750\(\pm\).0093 &\scriptsize.9570\(\pm\).0280\(\bullet\) &\scriptsize.9554\(\pm\).0261\(\bullet\) &\scriptsize.9860\(\pm\).0086\(\circ\) &\scriptsize.9401\(\pm\).0090\(\bullet\) &\scriptsize.9794\(\pm\).0081 &\scriptsize.7800\(\pm\).1429\(\bullet\) &\scriptsize .9647\(\pm\).0099 \\
\hline

\multirow{3}{*}{\scriptsize\textsf{usps}}
&\scriptsize(0.1, 0.2) &\scriptsize.9680\(\pm\).0054 &\scriptsize.9775\(\pm\).0034\(\circ\) &\scriptsize.9007\(\pm\).0099\(\bullet\) &\scriptsize.9782\(\pm\).0027\(\circ\) &\scriptsize.8993\(\pm\).0096\(\bullet\) &\scriptsize.8889\(\pm\).0074\(\bullet\) &\scriptsize.7896\(\pm\).0531\(\bullet\) &\scriptsize .9720\(\pm\).0048\\
&\scriptsize(0.3, 0.1) &\scriptsize.9604\(\pm\).0038 &\scriptsize.9692\(\pm\).0053\(\circ\) &\scriptsize.8699\(\pm\).0108\(\bullet\) &\scriptsize.9724\(\pm\).0039\(\circ\) &\scriptsize.8843\(\pm\).0080\(\bullet\) &\scriptsize.8530\(\pm\).0137\(\bullet\) &\scriptsize.6725\(\pm\).1158\(\bullet\) &\scriptsize .9514\(\pm\).0052\(\bullet\)\\
&\scriptsize(0.4, 0.4) &\scriptsize.8988\(\pm\).0160 &\scriptsize.7388\(\pm\).0227\(\bullet\) &\scriptsize.7437\(\pm\).0342\(\bullet\) &\scriptsize.9005\(\pm\).0124 &\scriptsize.7889\(\pm\).0298\(\bullet\) &\scriptsize.8148\(\pm\).0149\(\bullet\) &\scriptsize.6118\(\pm\).0634\(\bullet\) &\scriptsize .9154\(\pm\).0123\\
\hline

\multirow{3}{*}{\scriptsize\textsf{pendigits}}
&\scriptsize(0.1, 0.2) &\scriptsize.9927\(\pm\).0010 &\scriptsize.9965\(\pm\).0008 &\scriptsize.8360\(\pm\).0039\(\bullet\) &\scriptsize.9974\(\pm\).0007 &\scriptsize.8398\(\pm\).0042\(\bullet\) &\scriptsize.8371\(\pm\).0040\(\bullet\) &\scriptsize.8081\(\pm\).0228\(\bullet\) &\scriptsize .9926\(\pm\).0010\\
&\scriptsize(0.3, 0.1) &\scriptsize.9910\(\pm\).0023 &\scriptsize.9942\(\pm\).0014 &\scriptsize.8194\(\pm\).0083\(\bullet\) &\scriptsize.9955\(\pm\).0013 &\scriptsize.8359\(\pm\).0049\(\bullet\) &\scriptsize.8338\(\pm\).0057\(\bullet\) &\scriptsize.8089\(\pm\).0207\(\bullet\) &\scriptsize .9893\(\pm\).0021\\
&\scriptsize(0.4, 0.4) &\scriptsize.9472\(\pm\).0127 &\scriptsize.8505\(\pm\).0416\(\bullet\) &\scriptsize.6813\(\pm\).0518\(\bullet\) &\scriptsize.9523\(\pm\).0092 &\scriptsize.8086\(\pm\).0123\(\bullet\) &\scriptsize.8198\(\pm\).0053\(\bullet\) &\scriptsize.6606\(\pm\).0586\(\bullet\) &\scriptsize .9369\(\pm\).0156\\
\hline

\multirow{3}{*}{\scriptsize\textsf{letter}}
&\scriptsize(0.1, 0.2) &\scriptsize.9290\(\pm\).0052 &\scriptsize.7805\(\pm\).0049\(\bullet\) &\scriptsize.6754\(\pm\).0068\(\bullet\) &\scriptsize.7647\(\pm\).0049\(\bullet\) &\scriptsize.6748\(\pm\).0046\(\bullet\) &\scriptsize.6723\(\pm\).0076\(\bullet\) &\scriptsize.6686\(\pm\).0109\(\bullet\) &\scriptsize .9284\(\pm\).0066\\
&\scriptsize(0.3, 0.1) &\scriptsize.9219\(\pm\).0058 &\scriptsize.7743\(\pm\).0066\(\bullet\) &\scriptsize.6822\(\pm\).0080\(\bullet\) &\scriptsize.7593\(\pm\).0055\(\bullet\) &\scriptsize.6771\(\pm\).0075\(\bullet\) &\scriptsize.6229\(\pm\).0108\(\bullet\) &\scriptsize.6426\(\pm\).0152\(\bullet\) &\scriptsize .9161\(\pm\).0054\\
&\scriptsize(0.4, 0.4) &\scriptsize.7712\(\pm\).0099 &\scriptsize.6673\(\pm\).0438\(\bullet\) &\scriptsize.5862\(\pm\).0547\(\bullet\) &\scriptsize.7085\(\pm\).0142\(\bullet\) &\scriptsize.6663\(\pm\).0100\(\bullet\) &\scriptsize.6671\(\pm\).0096\(\bullet\) &\scriptsize.6298\(\pm\).0345\(\bullet\) &\scriptsize .7689\(\pm\).0071\\
\hline

\multirow{3}{*}{\scriptsize\textsf{magic04}}
&\scriptsize(0.1, 0.2)&\scriptsize.8315\(\pm\).0054 &\scriptsize.8136\(\pm\).0073\(\bullet\) &\scriptsize.7904\(\pm\).0051\(\bullet\) &\scriptsize.8171\(\pm\).0076\(\bullet\) &\scriptsize.7921\(\pm\).0040\(\bullet\) &\scriptsize.7935\(\pm\).0054\(\bullet\) &\scriptsize.7909\(\pm\).0075\(\bullet\) &\scriptsize .8291\(\pm\).0049\\
&\scriptsize(0.3, 0.1) &\scriptsize.8180\(\pm\).0044 &\scriptsize.8091\(\pm\).0030 &\scriptsize.7723\(\pm\).0045\(\bullet\) &\scriptsize.8121\(\pm\).0035 &\scriptsize.7902\(\pm\).0039\(\bullet\) &\scriptsize.7741\(\pm\).0045\(\bullet\) &\scriptsize.7619\(\pm\).0199\(\bullet\) &\scriptsize .8082\(\pm\).0023\(\bullet\)\\
&\scriptsize(0.4, 0.4) &\scriptsize.7767\(\pm\).0089 &\scriptsize.7340\(\pm\).0118\(\bullet\) &\scriptsize.7444\(\pm\).0186\(\bullet\) &\scriptsize.7666\(\pm\).0239 &\scriptsize.7813\(\pm\).0109 &\scriptsize.7493\(\pm\).0104\(\bullet\) &\scriptsize.7536\(\pm\).0290\(\bullet\) &\scriptsize .7823\(\pm\).0078\\
\hline

\multirow{3}{*}{\scriptsize\textsf{w8a}}
&\scriptsize(0.1, 0.2) &\scriptsize .9805\(\pm\).0015 &\scriptsize .9706\(\pm\).0015\(\bullet\) &\scriptsize .9845\(\pm\).0006 &\scriptsize .9786\(\pm\).0015 &\scriptsize .9588\(\pm\).0135\(\bullet\) &\scriptsize .8852\(\pm\).0030\(\bullet\) &\scriptsize .8695\(\pm\).0132\(\bullet\) &\scriptsize .9805\(\pm\).0014\\
&\scriptsize(0.3, 0.1) &\scriptsize .9807\(\pm\).0008 &\scriptsize .9708\(\pm\).0011\(\bullet\) &\scriptsize .9825\(\pm\).0012 &\scriptsize .9781\(\pm\).0016 &\scriptsize .9614\(\pm\).0127\(\bullet\) &\scriptsize .8897\(\pm\).0025\(\bullet\) &\scriptsize .8829\(\pm\).0089\(\bullet\) &\scriptsize .9803\(\pm\).0015\\
&\scriptsize(0.4, 0.4) &\scriptsize .9769\(\pm\).0073 &\scriptsize .9696\(\pm\).0012 &\scriptsize .9774\(\pm\).0012 &\scriptsize .9720\(\pm\).0011 &\scriptsize .9152\(\pm\).0524\(\bullet\) &\scriptsize .8377\(\pm\).0087\(\bullet\) &\scriptsize .7451\(\pm\).0349\(\bullet\) &\scriptsize .9528\(\pm\).0065\(\bullet\)\\
\hline

\multirow{3}{*}{\scriptsize\textsf{shuttle}}
&\scriptsize(0.1, 0.2) &\scriptsize.9967\(\pm\).0006 &\scriptsize.9559\(\pm\).0060\(\bullet\) &\scriptsize.9200\(\pm\).0117\(\bullet\)  &\scriptsize.9307\(\pm\).0035\(\bullet\) &\scriptsize.8108\(\pm\).0042\(\bullet\) &\scriptsize.8370\(\pm\).0060\(\bullet\) &\scriptsize.8402\(\pm\).0140\(\bullet\) &\scriptsize .9968\(\pm\).0008 \\
&\scriptsize(0.3, 0.1) &\scriptsize.9958\(\pm\).0006 &\scriptsize.9335\(\pm\).0029\(\bullet\) &\scriptsize.8339\(\pm\).0155\(\bullet\) &\scriptsize.9252\(\pm\).0032\(\bullet\) &\scriptsize.8099\(\pm\).0044\(\bullet\) &\scriptsize.8290\(\pm\).0039\(\bullet\) &\scriptsize.8385\(\pm\).0285\(\bullet\) &\scriptsize .9952\(\pm\).0008\\
&\scriptsize(0.4, 0.4) &\scriptsize.9550\(\pm\).0310 &\scriptsize.8415\(\pm\).0030\(\bullet\) &\scriptsize.8056\(\pm\).0030\(\bullet\) &\scriptsize.8451\(\pm\).0119\(\bullet\) &\scriptsize.8005\(\pm\).0119\(\bullet\) &\scriptsize.7987\(\pm\).0109\(\bullet\) &\scriptsize.8273\(\pm\).0250\(\bullet\) &\scriptsize .9696\(\pm\).0046\(\circ\)\\
\hline

\multirow{3}{*}{\scriptsize\textsf{acoustic}}
&\scriptsize(0.1, 0.2) &\scriptsize .7770\(\pm\).0012 &\scriptsize .7663\(\pm\).0033\(\bullet\) &\scriptsize .7547\(\pm\).0039\(\bullet\) &\scriptsize .7638\(\pm\).0036\(\bullet\) &\scriptsize .7619\(\pm\).0033\(\bullet\) &\scriptsize .7536\(\pm\).0028\(\bullet\) &\scriptsize .7151\(\pm\).0629\(\bullet\) &\scriptsize .7726\(\pm\).0016\(\bullet\)\\
&\scriptsize(0.3, 0.1) &\scriptsize .7700\(\pm\).0031 &\scriptsize .7629\(\pm\).0030 &\scriptsize .7477\(\pm\).0058\(\bullet\) &\scriptsize .7609\(\pm\).0030 &\scriptsize .7620\(\pm\).0025 &\scriptsize .7141\(\pm\).0043\(\bullet\) &\scriptsize .6553\(\pm\).0769\(\bullet\) &\scriptsize .7579\(\pm\).0028\(\bullet\)\\
&\scriptsize(0.4, 0.4) &\scriptsize .7575\(\pm\).0061 &\scriptsize .7396\(\pm\).0034\(\bullet\) &\scriptsize .6079\(\pm\).0998\(\bullet\) &\scriptsize .7445\(\pm\).0042\(\bullet\) &\scriptsize .7560\(\pm\).0034 &\scriptsize .7532\(\pm\).0034 &\scriptsize .5470\(\pm\).0888\(\bullet\) &\scriptsize .7111\(\pm\).0083\(\bullet\)\\
\hline

\multicolumn{3}{|c|}{\scriptsize\textsf{win/tie/loss}}   &\scriptsize\bf 35/20/5 &\scriptsize\bf 45/15/0 &\scriptsize\bf 28/25/7 &\scriptsize\bf 47/13/0 &\scriptsize\bf49/11/0 &\scriptsize\bf 53/7/0 &\scriptsize\bf 23/33/4 \\
\hline
\end{tabular}}\vspace{-0.1in}
\end{table*}
}

Notice that we directly take the true noise proportions as priors in the implementations of the first four algorithms \textsf{IR-KSVM}, \textsf{IR-LLog}, \textsf{LD-KSVM} and \textsf{UE-LLog}. For R\(k\)NN approach, however, we use  \(k'\)-nearest neighbor to make estimations of noise proportions \(\hat{\tau}_+\) and \(\hat{\tau}_-\) from the corrupted training datasets. Obviously, it is an unfair comparison for R\(k\)NN. The performances of the compared methods are evaluated by 10 trials of 4-fold cross validation with different random seeds, where the test accuracy is obtained by averaging over 40 runs, as summarized in Table~\ref{tab:result}.

It is evident that R\(k\)NN is better than other four non-kernel algorithms \textsf{IR-LLog}, \textsf{UE-LLog}, \textsf{AROW} and \textsf{NHERD}. The win/tie/loss counts show that R\(k\)NN is clearly superior to these non-kernel algorithms, as it wins for most times and never loses. It is also observable that R\(k\)NN is highly competitive to two kernel methods \textsf{IR-KSVM} and \textsf{LD-KSVM} on most datasets, and R\(k\)NN takes relatively stable performance while two kernel methods drop drastically as noise proportions increase. These observations validate the effectiveness of R\(k\)NN, and the intuitive explanation is that R\(k\)NN makes local corrections on a handful of totally misled examples, whereas the other methods on learning with noisy labels take global adjustments on loss functions, which may be sensitive to random noise. In comparisons with traditional \(k\)NN, our R\(k\)NN achieves better performance for asymmetric noises, and takes comparable performance for symmetric noise as expected.

Besides the frequent pairwise t-test shown in Table~\ref{tab:result}, we also consider Bayesian t-test \citep{Wang:Liu2016} to compare the performance of various algorithms, because our derivations of main results are based on a Bayesian framework. According to Bayesian t-test, the counts of win/tie/loss of our RkNN and compared methods are shown in Table~\ref{tab:Bayesian}. As we can see, Bayesian t-test takes better statistical support than frequent pairwise t-test to verify our proposed RkNN algorithm.

\begin{table*}
\caption{The counts of win/tie/loss of our RkNN and compared methods.}\label{tab:Bayesian}\vspace{0.05in}
\centering
\begin{tabular}{|c|c|c|c|c|c|c|c|c|}
\hline
\scriptsize our RkNN &\scriptsize IR-KSVM  &\scriptsize IR-LLog&\scriptsize LD-KSVM&\scriptsize UE-LLog&\scriptsize AROW &\scriptsize NHERD&\scriptsize kNN\\
\hline
\scriptsize win/tie/loss &\scriptsize39/16/5 &\scriptsize46/12/2 &\scriptsize32/20/8 &\scriptsize48/12/0 &\scriptsize51/8/1 &\scriptsize53/7/0 &\scriptsize 30/25/5 \\
\hline
\end{tabular}
\end{table*}

\begin{table*}
\caption{Average estimated noise proportions \((\hat{\tau}_+,\hat{\tau}_-)\) according to \(k'\)-nearest neighbor in Algorithm~\ref{alg1}.}\label{tab:noise}
\centering
\begin{tabular}{|c|c|c|c|c|c|}
\hline
\scriptsize\((\tau_{+}, \tau_{-})\) &\scriptsize\textsf{heart}&\scriptsize\textsf{ionosphere}&\scriptsize\textsf{housing}&\scriptsize\textsf{cancer} &\scriptsize\textsf{diabetes}\\
\hline
\scriptsize(0.1, 0.2)	&\scriptsize(.050, .143) &\scriptsize (.009, .251)&\scriptsize (.036, .081)&\scriptsize (.013, .091)&\scriptsize(.003, .201)\\
\scriptsize(0.3, 0.1)	&\scriptsize (.258, .039)	&\scriptsize(.154, .115)&\scriptsize(.173, .003)&\scriptsize(.132, .000)&\scriptsize(.142, .098)\\
\scriptsize(0.4, 0.4)	&\scriptsize (.232, .257)	&\scriptsize (.177, .282)&\scriptsize (.200, .198)&\scriptsize(.184, .183)&\scriptsize(.181, .211)\\
\hline
\hline

\scriptsize\((\tau_{+}, \tau_{-})\)&\scriptsize\textsf{vehicle}&\scriptsize\textsf{fourclass}&\scriptsize\textsf{german}&\scriptsize\textsf{splice}
&\scriptsize\textsf{optdigits}\\
\hline
\scriptsize(0.1, 0.2)	&\scriptsize(.005, .053)&\scriptsize(.004, .028) &\scriptsize(.123, .151)&\scriptsize(.107, .157)&\scriptsize(.000, .007)\\
\scriptsize(0.3, 0.1)	&\scriptsize(.126, .020)&\scriptsize(.152, .000)&\scriptsize(.363, .015)&\scriptsize(.292, .025)&\scriptsize(.146, .000)\\
\scriptsize(0.4, 0.4)	&\scriptsize(.196, .225)&\scriptsize(.195, .185)&\scriptsize(.243, .191)&\scriptsize(.215, .180)&\scriptsize(.160, .188)\\
\hline
\hline
\scriptsize\((\tau_{+}, \tau_{-})\) &\scriptsize\textsf{segment}&\scriptsize\textsf{landsat}&\scriptsize\textsf{mushroom}&\scriptsize\textsf{usps} &\scriptsize\textsf{pendigits}\\
\hline
\scriptsize(0.1, 0.2)	&\scriptsize(.001, .020)	&\scriptsize(.000, .014)&\scriptsize(.000, .008)&\scriptsize(.000, .011) &\scriptsize(.000, .015)\\
\scriptsize(0.3, 0.1)	&\scriptsize (.093, .000)	&\scriptsize (.082, .000)&\scriptsize (.071, .000)&\scriptsize 	(.083, .000)&\scriptsize (.095, .000)\\
\scriptsize(0.4, 0.4)	&\scriptsize (.168, .134)	&\scriptsize (.108, .093)&\scriptsize (.112, .104)&\scriptsize(.117, .112)&\scriptsize(.084, .084)\\
\hline
\hline

\scriptsize\((\tau_{+}, \tau_{-})\)&\scriptsize\textsf{letter}&\scriptsize\textsf{magic04}&\scriptsize\textsf{w8a}&\scriptsize\textsf{shuttle} &\scriptsize\textsf{acoustic}\\
\hline
\scriptsize(0.1, 0.2)&\scriptsize(.000, .008) &\scriptsize(.001, .009) &\scriptsize(.060, .105) &\scriptsize(.001, .017)&\scriptsize(.000, .055)\\
\scriptsize(0.3, 0.1)&\scriptsize(.104, .000)&\scriptsize(.093, .002)&\scriptsize(.220, .000)&\scriptsize(.195, .001)&\scriptsize(.128, .000)\\
\scriptsize(0.4, 0.4)&\scriptsize(.093, .087)&\scriptsize(.081, .070)&\scriptsize 	(.275, .251)&\scriptsize 	(.094, .113)&\scriptsize(.135, .130)\\
\hline
\end{tabular}
\end{table*}

Table~\ref{tab:noise} shows the average noise proportions estimated by R\(k\)NN on benchmark datasets. As we can see, the trend of true difference \(\tau_{+}-\tau_{-}\) can be observed from the estimated difference \(\hat{\tau}_{+}- \hat{\tau}_{-}\) in some way, though R\(k\)NN seldom makes precise estimation on noise proportions \(\tau_+\) and \(\tau_-\), particularly for large datasets and small noise proportions. It is also noticed that the R\(k\)NN approach achieves good performance, as shown in Table~\ref{tab:result}, even for rather rough estimation on noise proportions. Those observations further validate the robustness of the R\(k\)NN approach.

\begin{figure}[!t]
\centering
\begin{minipage}{5.4in}
\includegraphics[width=5.4in]{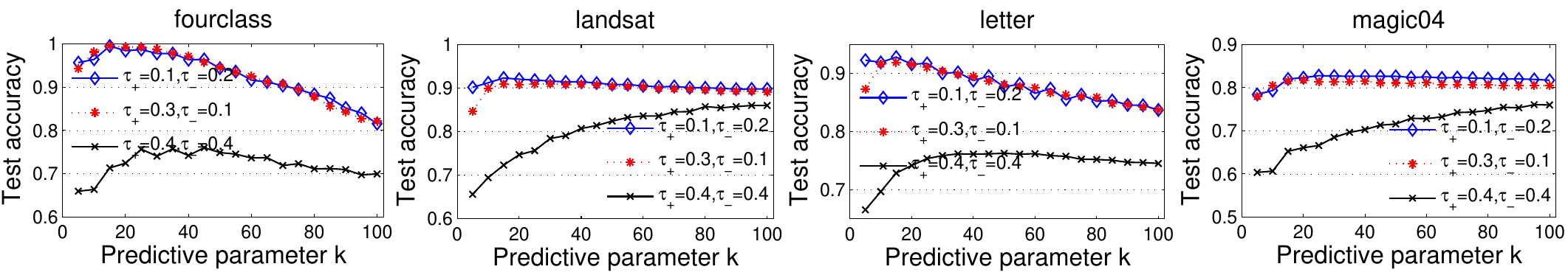}\vspace{-0.3in}\\
\caption{Influence of predictive parameter \(k\)}\label{fig:parak}\vspace{0.1in}
\end{minipage}

\begin{minipage}{5.4in}
\includegraphics[width=5.4in]{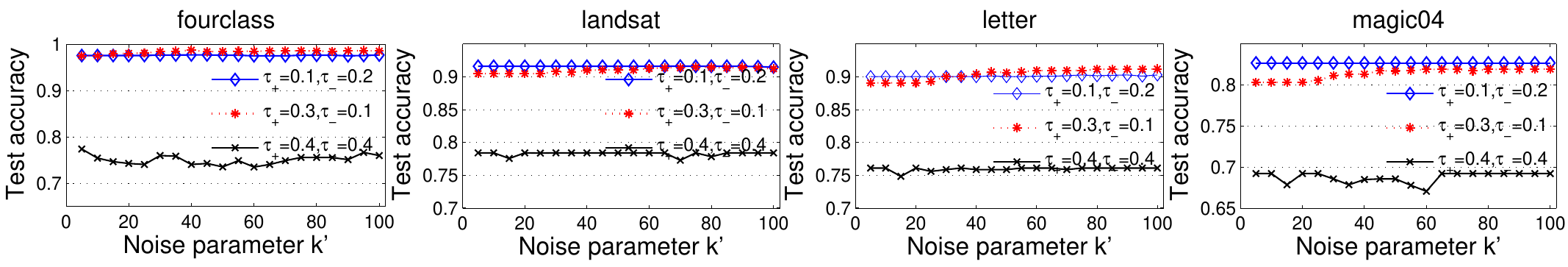}\vspace{-0.3in}\\
\caption{Influence of noise parameter \(k'\)}\label{fig:parak1}\vspace{0.1in}
\end{minipage}

\begin{minipage}{5.4in}
\includegraphics[width=5.4in]{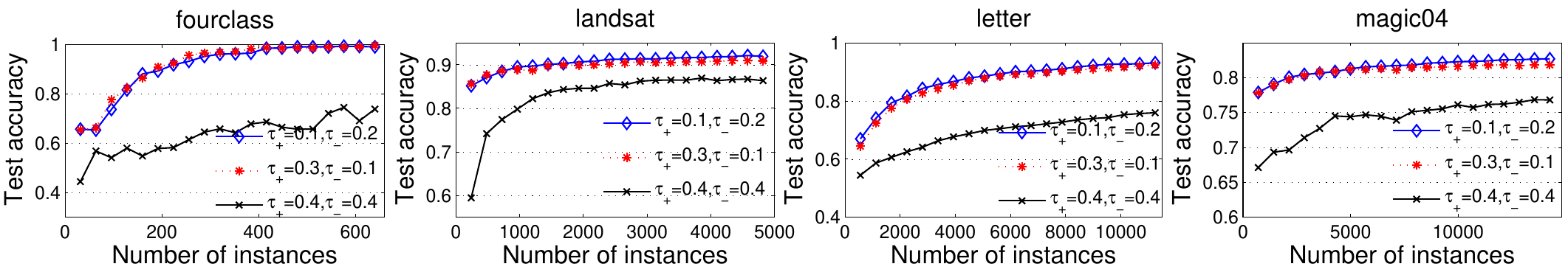}\vspace{-0.3in}\\
\caption{Influence of sample size}\label{fig:convergence}\vspace{0.1in}
\end{minipage}\vspace{-0.1in}
\end{figure}

\subsection{Parameter Influence}\label{sec:exp:para}
We investigate the influence of parameters in this section.  Figure~\ref{fig:parak} shows that R\(k\)NN is not sensitive to the values of predictive parameter \(k\) given that it is not set smaller than 10, and we'd better take large \(k\) when tackle large datasets and high noise proportions. Figure~\ref{fig:parak1} shows that the noise parameter \(k'\) should not be set to value smaller than 20, and there is a relative big range between 20 and 100 where R\(k\)NN achieves better performance. Figure~\ref{fig:convergence} shows the convergence of performance as sample size increases, which illustrates that R\(k\)NN takes stable and convergent performance as expected. Relevant analysis also shows the robustness of R\(k\)NN. Here, we  present empirical analysis of parameters  on four datasets, while the trends are similar on other datasets.

\section{Conclusion}\label{sec:con}
This work presents the finite-sample and distribution-dependent bounds on the consistency of nearest neighbor. The theoretical results show that, for asymmetric noises, \(k\)-nearest neighbor is robust enough to classify most data correctly, except for a handful of examples, whose labels are totally misled by random noises. For symmetric noises, however, \(k\)-nearest neighbor achieves the same consistent rate as that of noise-free setting, which verifies the resistance of \(k\)-nearest neighbor. Motivated from theoretical analysis, we propose the Robust \(k\)-Nearest Neighbor (R\(k\)NN) approach to deal with noisy labels. The basic idea is to make unilateral corrections to examples, whose labels are totally misled by random noises, and classify the others directly by utilizing the robustness of \(k\)-nearest neighbor. Extensive experiments validate the effectiveness of the proposed R\(k\)NN method. An interesting future work is to develop robust \(k\)-nearest neighbor for large-scale and high-dimensional datasets in the random noise setting.

{\bibliography{reference}\bibliographystyle{apalike}}
\end{document}